\documentclass{article} 
\usepackage{iclr2019_conference,times}

\usepackage{amsmath,amsfonts,bm}

\def\eqref#1{equation~\ref{#1}}

\def\floor#1{\lfloor #1 \rfloor}
\def\1{\bm{1}}

\DeclareMathAlphabet{\mathsfit}{\encodingdefault}{\sfdefault}{m}{sl}
\SetMathAlphabet{\mathsfit}{bold}{\encodingdefault}{\sfdefault}{bx}{n}

\newcommand{\R}{\mathbb{R}}

\DeclareMathOperator*{\argmin}{arg\,min}

\usepackage{hyperref}
\usepackage{url}

\usepackage{wrapfig}
\usepackage{graphicx} 
\usepackage{amsmath}
\usepackage{tikz}
\usepackage{pgfplots}
\usepackage{natbib}

\usepackage[utf8]{inputenc} 
\usepackage[T1]{fontenc}    
\usepackage{hyperref}       
\usepackage{url}            
\usepackage{booktabs}       
\usepackage{amsfonts}       
\usepackage{nicefrac}       
\usepackage{microtype}      
\usepackage{float}
\usepackage{tikz}
\usetikzlibrary{calc}
\usetikzlibrary{patterns}
\usepackage{pgfplots}
\usepackage{verbatim}
\usepackage{newtxtext}

\tikzset{
    ncbar angle/.initial=90,
    ncbar/.style={
        to path=(\tikztostart)
        -- ($(\tikztostart)!#1!\pgfkeysvalueof{/tikz/ncbar angle}:(\tikztotarget)$)
        -- ($(\tikztotarget)!($(\tikztostart)!#1!\pgfkeysvalueof{/tikz/ncbar angle}:(\tikztotarget)$)!\pgfkeysvalueof{/tikz/ncbar angle}:(\tikztostart)$)
        -- (\tikztotarget)
    },
    ncbar/.default=0.5cm,
}

\tikzset{square left brace/.style={ncbar=0.5cm}}
\tikzset{square right brace/.style={ncbar=-0.5cm}}

\newcommand{\reals}{\mathbb{R}}

\usepackage{amsmath,amsbsy,amsfonts,amssymb,dsfont,units,psfrag}
\usepackage{amssymb}
\usepackage{mathtools}
\usepackage{algorithm,algorithmic,mathtools}
\usepackage{enumitem}
\usepackage{color,cases,caption}

\usepackage{subcaption}
\usepackage{float}
\usepackage{tikz,lipsum}
\usetikzlibrary{calc, shapes, backgrounds, arrows, fit, calc, positioning}

 \newcommand{\IGNORE}[1]{}

\newcommand*{\eqdef}{\stackrel{\textup{def}}{=}}

\newcommand{\circrm}{\mathrm{circ}}
\newcommand{\vecrm}{{\mathrm{vec}}}

 \def\0{{\bf 0}}

\def\qed{\hfill\hbox{${\vcenter{\vbox{
    \hrule height 0.4pt\hbox{\vrule width 0.4pt height 6pt
    \kern5pt\vrule width 0.4pt}\hrule height 0.4pt}}}$}}

\definecolor{dkr}{rgb}{0.6,0.2,0.2}
\definecolor{dkg}{HTML}{00CC00}
\definecolor{dkb}{rgb}{0.0,0.3,0.7}
\definecolor{blue1}{HTML}{0066FF}
\definecolor{brm}{rgb}{1,0.0,1}
\definecolor{lpurple}{cmyk}{.05,0.18,0,0}
\definecolor{dred}{cmyk}{0,1,.81,.04}
\definecolor{lred}{cmyk}{0,0.26,.1,0}
\definecolor{dgreen}{cmyk}{1,0,.53,.21}
\definecolor{lgreen}{cmyk}{.38,0,.36,0}
\definecolor{dblue}{cmyk}{1,0.44,0,0}
\definecolor{lblue}{cmyk}{.25,0.02,0,0}

\def\beq{\begin{equation}}
\def\eeq{\end{equation}\noindent}
\def\beqn{\begin{eqnarray}}
\def\eeqn{\end{eqnarray} \noindent}
\def\beqnn{  \begin{eqnarray*}}
\def\eeqnn{\end{eqnarray*}  \noindent}
\def\bcase{  \begin{numcases}}
\def\ecase{\end{numcases}   \noindent}
\def\bsbcase{  \begin{subnumcases}}
\def\esbcase{\end{subnumcases}   \noindent}

\newtheorem{theorem}{Theorem}

\newtheorem{lemma}[theorem]{Lemma}

\newtheorem{proposition}[theorem]{Proposition}

\newenvironment{proof}{\noindent{\bf Proof}:}{\qed}

\newcommand{\matplottc}[1]{               
        \unitlength .45truein
        \begin{center}
        \includegraphics{#1.ps}
        \end{picture}
        \end{center}
}

\def\psfancypar#1#2{\begingroup\def\par{\endgraf\endgroup\lineskiplimit=0pt}
               \setbox2=\hbox{\large\sc #2}
               \newdimen\tmpht \tmpht \ht2 \advance\tmpht by \baselineskip
               \font\hhuge=Times-Bold at \tmpht
               \setbox1=\hbox{{\hhuge #1}}
               \count7=\tmpht \count8=\ht1
               \divide\count8 by 1000 \divide\count7 by \count8
               \tmpht=.001\tmpht\multiply\tmpht by \count7
               \font\hhuge=Times-Bold at \tmpht
               \setbox1=\hbox{{\hhuge #1}}
               \noindent
                \hangindent1.05\wd1
               \hangafter=-2 {\hskip-\hangindent
               \lower1\ht1\hbox{\raise1.0\ht2\copy1}%
                \kern-0\wd1}\copy2\lineskiplimit=-1000pt}

\def\Kout{\setbox1=\hbox{\Huge\bf K}\hbox to
1.05\wd1{\hspace{.05\wd1}
}}
\def\Sout{\setbox1=\hbox{\Huge\bf S}\hbox to 1.05\wd1{\hspace{.05\wd1}
}}

\newcommand{\torestate}[3]{%
\expandafter \def \csname BBRESTATE #2 \endcsname{#3}
\theoremstyle{plain}
\newtheorem{BBRESTATETHMNUM#2}[theorem]{#1}
\begin{BBRESTATETHMNUM#2}\label{#2}\csname BBRESTATE #2 \endcsname   \end{BBRESTATETHMNUM#2}
\newtheorem*{BBRESTATETHMNONNUM#2}{{#1}~\ref{#2}}
}

\newcommand{\restate}[1]{\begin{BBRESTATETHMNONNUM#1}[Restated] \csname BBRESTATE #1 \endcsname
\end{BBRESTATETHMNONNUM#1}}

\definecolor{blue1}{HTML}{0066FF}
\definecolor{lpurple}{cmyk}{.05,0.18,0,0}

\title{\bf The Singular Values of Convolutional Layers}

\author{
\noindent
\hspace*{-5pt}
Hanie Sedghi, Vineet Gupta and Philip M. Long \\
Google Brain\\
Mountain View, CA 94043\\
\texttt{\{hsedghi,vineet,plong\}@google.com} \\
}

 \date{} 

\iclrfinalcopy 

\begin{document}

\maketitle

\begin{abstract}
We characterize the singular values of the linear transformation associated with
a standard 2D multi-channel convolutional layer, enabling their
efficient computation. 
This characterization also leads to an algorithm for projecting a convolutional layer onto 
an operator-norm ball.
We show that this is an effective regularizer; 
for example, it improves the test error of a deep residual network 
using batch normalization
on CIFAR-10 from 6.2\% to 5.3\%. 
\end{abstract}

\section{Introduction}

Exploding and vanishing gradients \citep{hochreiter1991untersuchungen,hochreiter2001gradient,goodfellow2016deep}
are fundamental obstacles to effective training of deep neural networks.  Many deep networks used in practice
are layered.  We can think of such networks as the composition of a number of feature transformations, followed
by a linear classifier on the final layer of features.  The singular values of the Jacobian of a layer bound
the factor by which it increases or decreases the norm of the backpropagated signal.  If these singular values
are all close to 1, then gradients neither explode nor vanish.  These singular values also bound these
factors in the forward direction, which affects the stability of computations, including whether the network
produces the dreaded ``Nan''.  Moreover, it has been proven~\citep{bartlett2017spectrally} that the generalization error for a network is bounded by the Lipschitz constant of the network, 
which in turn can be bounded by the product of the operator
norms of the Jacobians
of the layers.
\citet{cisse2017parseval} 
discussed robustness to adversarial examples as a result of bounding the operator norm. 

These considerations have led authors to regularize networks by driving down
the operator norms of network layers \citep{drucker1992improving, hein2017formal, yoshida2017spectral,miyato2018spectral}. 
Orthogonal initialization \citep{saxe2013exact,pennington2017resurrecting} and Parseval networks~\citep{cisse2017parseval} are motivated by similar considerations.

Convolutional layers \citep{lecun1998gradient} are key components of modern deep networks.  They compute linear transformations of their inputs.
The Jacobian of a linear transformation is always equal to the linear transformation itself.   Because of the central
importance of convolutional layers to the practice of deep learning, and the fact that the singular values of the linear
transformation computed by a convolutional layer are the key to its contribution to exploding and vanishing
gradients, we study these singular values. 
Up until now, authors seeking to control the operator norm
of convolutional layers
have resorted to approximations \citep{yoshida2017spectral,miyato2018spectral,gouk2018regularisation}. 
In this paper, we provide an efficient way to compute the singular values exactly --- this opens the door to various regularizers. 

We consider the convolutional layers commonly applied to image analysis tasks.  The input to a typical layer is
a feature map, with multiple channels for each position in an $n \times n$ field.  If there are $m$ channels,
then the input as a whole is a $m \times n \times n$ tensor.  The output is also an $n\times n$ field with
multiple channels per position\footnote{Here, to keep things simple, we are concentrating on the case that the
stride is $1$.}.  Each channel of the output
is obtained by taking a linear combination of the values of the features in all channels in a local neighborhood
centered at the corresponding position in the input feature map.  Crucially, the same linear combination is
used for all positions in the feature map.  The coefficients are compiled in the {\em kernel} of
the convolution.  If the neighborhood is a $k \times k$ region, a kernel $K$ is a $k \times k \times m \times m$ tensor.
The projection $K_{:,:,c,:}$ gives the coefficients that determine the $c$th channel of the output, in terms of
the values found in all of the channels of all positions in it neighborhood; $K_{:,:,c,d}$ gives the coefficients to
apply to the $d$th input channel, and $K_{p,q,c,d}$ is the coefficient to apply to this input at in the position in the field offset horizontally by $p$ and vertically by $q$. For ease of exposition, we assume that feature maps and local neighborhoods are square and that the 
number of channels in the output is equal to the number of channels in the input - the extension to the general case is completely straightforward.

To handle edge cases in which the offsets call for inputs that are off 
the feature maps, practical convolutional
layers either do not compute outputs (reducing the size of the feature map), or pad the input with zeros.  The behavior of
these layers can be approximated by a layer that treats the input as if it were a torus; when the offset calls for a pixel that is off the right end of the image, the layer ``wraps around'' to take it from the left edge,
and similarly for the other edges.
The quality of this approximation has been heavily analyzed in the case of one-dimensional signals \citep{gray2006toeplitz}.
Consequently, theoretical analysis of convolutions that wrap around has been become standard.  This is the case analyzed
in this paper.


\paragraph{Summary of Results: }
Our main result is a characterization of the singular values of a convolutional layer in terms of the kernel
tensor $K$. 
Our characterization enables these singular values to be computed exactly in a simple and practically
fast way, using $O(n^2 m^2 (m + \log n))$ time.  For comparison, the brute force solution that performs SVD on the matrix that encodes the convolutional layer's linear transformation would take $O((n^2 m)^3) = O(n^6 m^3)$ time, and is 
impractical for commonly used network sizes.  As another point of comparison, simply to compute the convolution
takes $O(n^2 m^2 k^2)$ time.  We prove that the following two lines of NumPy correctly compute the singular values.
{\small
\begin{verbatim}
  def SingularValues(kernel, input_shape):
    transforms = np.fft.fft2(kernel, input_shape, axes=[0, 1])
    return np.linalg.svd(transforms, compute_uv=False)
\end{verbatim}
}
Here \verb+kernel+ is any 
$k \times k \times m \times m$ tensor\footnote{The same code
also works if the filter height is different from the
filter width, and if the number of channels in the input is
different from the number of channels of the output.}
and \verb+input_shape+ is the shape of the feature map to be convolved. A TensorFlow implementation is similarly simple.

Timing tests, reported in Section~\ref{s:timing}, confirm that this characterization speeds up the computation
of singular values by multiple orders of magnitude -- making it usable in practice.   The algorithm
first performs $m^2$ FFTs, and then it performs $n^2$ SVDs.  The FFTs, and then the SVDs, may be
executed in parallel.  Our TensorFlow implementation runs a lot faster than the
NumPy implementation (see Figure~\ref{fig:timing});  
we think that this parallelism is the cause. We used our
code to compute the singular values of the convolutional layers of the official ResNet-v2 model released with
TensorFlow~\citep{he2016identity}.  The results are described in Appendix~\ref{a:resnet}. 

Exposing the singular values of a convolutional layer opens the door to a variety of regularizers for these layers, including
operator-norm regularizers. 
In Section~\ref{s:cifar}, we evaluate an algorithm that periodically projects each convolutional layer onto a
operator-norm ball.
Using the projections improves the test error from 6.2\% to 5.3\% on CIFAR-10. We evaluate bounding the operator norm with and without batchnorm and we see that regularizing the operator norm helps, even in the presence of batch
normalization. Moreover, operator-norm regularization and batch normalization
are not redundant, and neither dominates the other. They complement each other. 

\paragraph{Related work:}\label{s:related}
Prior to our work, authors have responded to the difficulty
of computing the singular values of convolutional layers in various ways.
\citet{cisse2017parseval} constrained the matrix to have orthogonal rows and scale the output of each layer by a factor of $(2k+1)^{-\frac{1}{2}}$, for $k \times k$ kernels. 
\citet{gouk2018regularisation, gouk2018maxgain} proposed regularizing using a per-mini-batch approximation to the operator norm.  They find the largest ratio between the input
and output of a layer in the minibatch, and then scale down the
transformation (thereby scaling down all of the
singular values, not just the largest ones) 
so that the new value of this ratio obeys a constraint.

\citet{yoshida2017spectral} used an approximation of the operator norm of
a reshaping of $K$ in place of the operator norm for the linear transformation associated with $K$ in
their experiments.  They reshape the given $k \times k \times m \times m$ into a $mk^2 \times m$ matrix, and
compute its {\em largest} singular value using a power iteration method, and use this as a substitute for the 
operator norm. While this provides a useful heuristic for regularization, the largest singular value of the reshaped matrix is often quite different from the operator norm of the linear transform associated with $K$. 
Furthermore if we want to regularize using projection onto an operator-norm ball, we need the whole spectrum of the linear transformation
(see Section~\ref{sec:regularization}). The reshaped $K$ has only $m$ singular values, whereas the linear transformation
has $mn^2$ singular values of which $mn^2/2$ are distinct except in rare degenerate cases. It is possible to project the reshaped $K$ 
onto an operator-norm ball
by taking its SVD and clipping its singular values --- we conducted experiments with this projection and report the results in Section~\ref{s:yoshida}.

A close relative of our main result was independently discovered
by \citet[Lemma 2]{bibi2018deep}.

\paragraph{Overview of the Analysis: }
If the signal is 1D and there is a single input and output channel, then the linear transformation
associated with a convolution is encoded by a {\em circulant matrix}, i.e., a matrix whose rows are circular
shifts of a single row \citep{gray2006toeplitz}.  For example, for a row $a = (a_1, a_2, a_3)$, the circulant matrix
$\circrm(a)$ generated by $a$ is 
$
\left(
\begin{array}{ccc}
    a_0 & a_1 & a_2 \\
    a_2 & a_0 & a_1 \\
    a_1 & a_2 & a_0 
\end{array}
\right).
$
In the special case of a 2D signal with a single input channel and single output channel, the
linear transformation is {\em doubly block circulant} (see \citep{goodfellow2016deep}).  Such a matrix 
is made up of a circulant matrix of blocks, each of which in turn is itself circulant.  Finally, when
there are $m$ input channels and $m$ output channels, there are three levels to the hierarchy: there is
a $m \times m$ matrix of blocks, each of which is doubly block circulant.  Our analysis extends tools from
the literature built for circulant \citep{horn2013matrix} and doubly circulant \citep{chao1974note} matrices to
analyze the matrices with a third level in the hierarchy arising from the convolutional layers used in
deep learning.  One key point is that the eigenvectors of a circulant matrix are Fourier basis vectors: in the
2D, one-channel case, the matrix whose columns are the eigenvectors is $F \otimes F$, for the matrix $F$ whose
columns form the Fourier basis.  Multiplying by this matrix is a 2D Fourier transform.
In the multi-channel case, we show that the singular values can be computed
by (a) finding the eigenvalues of each of the $m^2$  doubly circulant matrices (of dimensions $n^2 \times n^2$) using a 2D Fourier transform, (b)
by forming multiple $m \times m$ matrices, one for each eigenvalue, by picking out
 the $i$-th eigenvalue of each of the $n^2\times n^2$ blocks, for $i \in [1..n^2]$.  The union of all of the singular values of all of those $m \times m$ matrices is the multiset of singular values of the layer.

 \textbf{Notation: }
We use upper case letters for matrices, lower case for vectors. For matrix $M$, $M_{i,:}$ represents the $i$-th row and $M_{:,j}$ represents the $j$-th column; we will also use the analogous notation for higher-order tensors. 
The operator norm of $M$ is denoted by $|| M ||_2$.
For $n \in \mathbb{N}$, we use $[n]$ to denote the set $\lbrace 0,1, \dotsc, n-1 \rbrace$
(instead of usual $\lbrace 1, \dotsc, n \rbrace$).  We will index the rows and columns of
matrices using elements of $[n]$, i.e.\ numbering from $0$.  Addition of row and column
indices will be done mod $n$ unless otherwise indicated.  (Tensors will be treated analogously.)
Let $\sigma(\cdot)$ be the mapping from a matrix to (the multiset of) its singular values.
\footnote{
For two multisets $\cal S$ and
$\cal T$, we use ${\cal S} \cup {\cal T}$ to denote the multiset obtained by summing the multiplicities
of members of ${\cal S}$ and ${\cal T}$.} 


Let $\omega = \exp(2\pi i/n)$, where $i = \sqrt{-1}$.  (Because we
need a lot of indices in this paper, our use of $i$ to define $\omega$
is the only place in the paper where we will use $i$ to denote
$\sqrt{-1}$.)

Let $F$ be the $n \times n$ matrix that represents the discrete Fourier transform: 
$F_{ij} = \omega^{i j}$.
We use $I_n$ to denote the identity matrix of size $n \times n$. 
For $i \in [n]$, we use $e_i$ to represent the $i$th basis vector in $\mathbb{R}^n$. We use $\otimes$ to represent the Kronecker product between two matrices (which also refers to the outer product of two vectors).

\section{Analysis}

\subsection{One filter}
\label{s:one.filter}

As a warmup, we focus on the case that the number $m$ of input
channels and output channels is $1$.  In this case, the filter
coefficients are simply a $k \times k$ matrix.  It will simplify
notation, however, if we embed this $k \times k$ matrix in an
$n \times n$ matrix, by padding with zeroes (which corresponds
to the fact that the offsets with those indices are not used).
Let us refer to this $n \times n$ matrix also as $K$.

An $n^2 \times n^2$ matrix $A$ is {\em doubly block circulant} 
if 
$A$ is a circulant matrix of $n \times n$ blocks that are in turn circulant.

For a matrix $X$, let $\vecrm(X)$ be the vector obtained by stacking the columns of $X$.

\begin{lemma}[see \citet{jain1989fundamentals} Section 5.5, \citet{goodfellow2016deep} page 329]
\label{l:dbc}
For any filter coefficients $K$, the linear transform for the convolution by $K$
is represented by the following doubly block circulant matrix:
\begin{align}
\label{eqn:A}
A  = \left[  \begin{array}{cccc}
 \circrm(K_{0,:}) & \circrm(K_{1,:}) & \dotsc & \circrm(K_{n-1,:}) \\
 \circrm(K_{n-1,:}) & \circrm(K_{0,:}) & \dotsc & \circrm(K_{n-2,:}) \\
\vdots & \vdots & \vdots & \vdots \\
 \circrm(K_{1,:}) & \circrm(K_{2,:}) & \dotsc & \circrm(K_{0,:})
\end{array} \right].
\end{align}
That is,
if $X$ is an $n \times n$ matrix, and $Y$ is the result of a 2-d convolution of $X$ with $K$, i.e.\ 
\begin{align} \label{eqn:circular-conv}
\forall ij,\;    Y_{ij} =  \sum_{p \in [n]}\sum_{q \in [n]} X_{i+p,j+q} K_{p,q}
\end{align}
then $\vecrm(Y) = A ~\vecrm(X)$.
\end{lemma}

So now we want to determine the singular values of a doubly block circulant matrix.

We will make use of the characterization of the eigenvalues and eigenvectors of doubly block circulant matrices,
which uses the following definition:
$Q \eqdef \frac{1}{n} \left( F \otimes F \right).$
\begin{theorem}[\citet{jain1989fundamentals} Section 5.5] 
\label{thm:eigenvectors}
For any $n^2 \times n^2$ doubly block circulant matrix $A$,
the eigenvectors of $A$ are the columns of $Q$.
\end{theorem}
To get singular values in addition to eigenvalues, we need the following two lemmas.
\begin{lemma}[\citet{jain1989fundamentals} Section 5.5]
\label{l:Q.prop}
$Q$ is unitary.
\end{lemma}
Using Theorem~\ref{thm:eigenvectors} and Lemma~\ref{l:Q.prop}, we can get the eigenvalues as the
diagonal elements of $Q^* A Q$.

\begin{lemma}
\label{l:normal}
The matrix $A$ defined in \eqref{eqn:A} is normal, i.e., $A^T A = A A^T$.
\end{lemma}
\begin{proof}
\[
A A^T = A A^* = Q^* D Q Q^* D^* Q = Q^* D D^* Q = Q^* D^* D Q = Q^* D^* Q Q^* D Q = A^* A = A^T A.
\]
\end{proof}

The following theorem characterizes the singular values of $A$ as a simple function of
$K$.  As we will see, a characterization of the eigenvalues plays a major role.
\citet{chao1974note} provided a more technical characterization of the eigenvalues which may be regarded as
making partial progress toward Theorem~\ref{thm:singular_values}.  However, we provide a
proof from first principles, since it is the cleanest way we know to prove
the theorem.
\begin{theorem} \label{thm:singular_values}
For the matrix $A$ defined in \eqref{eqn:A},
the eigenvalues of $A$ are the entries of
$F^T K F$, and its singular values are their magnitudes.
That is, the singular values of $A$ are
\begin{equation} 
\label{eqn:singular-values}
    \left\{ \left| (F^T K F)_{u,v} \right|
            \;:\;  u, v \in [n]
            \right\}.
\end{equation}
\end{theorem}
\begin{proof}
By Theorems~\ref{thm:eigenvectors} and Lemma~\ref{l:Q.prop}, the eigenvalues of $A$ are the diagonal elements
of $Q^* A Q = 
    \frac{1}{n^2} (F^* \otimes F^*) A (F \otimes F)$.
If we view $(F^* \otimes F^*) A (F \otimes F)$ as a compound $n \times n$
matrix of $n \times n$ blocks, for $u,v \in [n]$, the 
$(u n + v)$th diagonal element is the $v$th element of 
the $u$th diagonal block.  Let us first evaluate the $u$th diagonal
block.  Using $i, j$ to index blocks, we have
\begin{align}
\nonumber
& (Q^* A Q)_{uu}
= \frac{1}{n^2} \sum_{i,j \in [n]} (F^* \otimes F^*)_{u i} 
   A_{ij} (F \otimes F)_{ju} 
=  
 \frac{1}{n^2}  
   \sum_{i,j \in [n]} \omega^{-ui} F^* \circrm(K_{j-i,:}) \omega^{j u} F \\
\label{e:diagonal_block}
& =  
 \frac{1}{n^2}  
\sum_{i,j \in [n]} \omega^{u (j - i)} F^* \circrm(K_{j-i,:}) F.
\end{align}
To get the $v$th element of the diagonal of (\ref{e:diagonal_block}), we may sum the
$v$th elements of the diagonals of each of its terms.  Toward this end, we have
\begin{align*}
& (F^* \circrm(K_{j-i,:}) F)_{vv} 
= \sum_{r,s \in [n]} \omega^{-vr} 
               \circrm(K_{j-i,:})_{rs} \omega^{s v} 
= \sum_{r,s \in [n]} \omega^{v(s-r)} 
          K_{j-i,s-r}. \\
\end{align*}
Substituting into (\ref{e:diagonal_block}), we get
$
 \frac{1}{n^2}  
\sum_{i,j,r,s \in [n]} \omega^{u (j - i)} \omega^{v(s-r)} K_{j-i,s-r}.
$
Collecting terms where $j-i=p$ and $s-r=q$, this is
$
\sum_{p,q \in [n]} \omega^{u p} \omega^{v q} K_{p,q}
 = (F^T K F)_{u v}.
$

Since the singular values of any normal matrix are the magnitudes of its
eigenvalues (\citet{horn2013matrix} page 158), applying Lemma~\ref{l:normal} completes the proof.
\end{proof}

Note that $F^T K F$ is the 2D Fourier transform of $K$, and recall that
$|| A ||_2$ is the largest singular value of $A$.

\subsection{Multi-channel convolution}

Now, we consider case where the number $m$ of channels may be more than one.
Assume we have a 4D kernel tensor $K$ with element $K_{p,q,c,d}$
giving the connection strength between a unit in channel $d$ of the
input and a unit in channel $c$ of the output, with an offset of $p$
rows and $q$ columns between the input unit and the output unit.  The
input $X \in \mathbb{R}^{m \times n \times n}$; element $X_{d,i,j}$ is
the value of the input unit within channel $d$ at row $i$ and column
$j$.  The output $Y \in \mathbb{R}^{m \times n \times n }$ has the
same format as $X$, and is produced by
\begin{align}
\label{eqn:circular-conv-multi}
    Y_{c r s} = \sum_{d \in [m]} \sum_{p \in [n]} \sum_{q \in [n]} X_{d, r + p, s + q} K_{p,q,c,d}.
\end{align} 

By inspection,
$\vecrm(Y) = M \vecrm(X)$, where $M$ is as follows
\begin{align}
\label{eqn:multi-channel-M}
M  = \left[  \begin{array}{cccc}
 B_{00} & B_{01} & \dotsc & B_{0(m-1)} \\
B_{10} & B_{11} & \dotsc & B_{1(m-1)}\\
\vdots & \vdots & \dotsc & \vdots   \\
B_{(m-1)0} & B_{(m-1)1} & \dotsc & B_{(m-1)(m-1)}
\end{array} \right]
\end{align}
and each $B_{cd}$ is a doubly block circulant matrix
from Lemma~\ref{l:dbc}
corresponding to the portion
$K_{:,:,c,d}$ of $K$ that concerns the effect of the $d$-th input channel on the $c$-th output channel. 
(We can think of each output in the multichannel case as being a sum of single channel
filters parameterized by one of the $K_{:,:,c,d}$'s.)


The following is our main result.  
\begin{theorem}
\label{thm:conv.sv}
For any $K \in \R^{n \times n \times m \times m}$, let $M$ is the
matrix encoding the linear transformation computed by a convolutional
layer parameterized by $K$, defined as in (\ref{eqn:multi-channel-M}).
For each $u,v \in [n] \times [n]$, let $P^{(u,v)}$ be
the $m \times m$ matrix given by $P^{(u,v)}_{cd} = (F^T K_{:,:,c,d} F)_{uv}$.
Then
\begin{equation}
\label{e:sigmaM}
\sigma(M)
 = 
\bigcup_{u \in [n], v \in [n]} \sigma \left(P^{(u,v)}
             \right).
\end{equation}
\end{theorem}

%
The rest of this section is devoted to proving Theorem~\ref{thm:conv.sv} through a series of lemmas.

The analysis of Section~\ref{s:one.filter} implies that for all $c, d \in [m]$, $D_{cd} \eqdef Q^* B_{cd} Q$ is diagonal.
Define
\begin{align}
\label{eqn:L}
L  \eqdef \left[  \begin{array}{cccc}
 D_{00} & D_{01} & \dotsc & D_{0(m-1)} \\
D_{10} & D_{11} & \dotsc & D_{1(m-1)}\\
\vdots & \vdots & \dotsc & \vdots   \\
D_{(m-1)0} & D_{(m-1)1} & \dotsc & D_{(m-1)(m-1)}
\end{array} \right].\end{align}

\begin{lemma} \label{lemma:m-l-svalue}
$M$ and $L$ have the same singular values.
\end{lemma}
\begin{proof}
We have
\begin{align*} 
M  & = \left[  
         \begin{array}{ccc}
            B_{00} & \dotsc & B_{0(m-1)} \\
           \vdots & \vdots & \vdots \\
           B_{(m-1)0} & \dotsc & B_{(m-1)(m-1)}
         \end{array} 
           \right] 
     = \left[  \begin{array}{ccc}
            Q D_{00} Q^*  & \dotsc & Q D_{0(m-1)} Q^* \\
            \vdots & \vdots & \vdots    \\
             Q D_{(m-1)0} Q^*  & \dotsc & Q D_{(m-1)(m-1)} Q^*
           \end{array} \right] \\
    & = R \left[  \begin{array}{ccc}
                  D_{00}   & \dotsc & D_{0(m-1)} \\
                  \vdots & \vdots & \vdots   \\
                   D_{(m-1)0}  & \dotsc & D_{(m-1)(m-1)}
                 \end{array} \right] R^* 
      = R L R^*, 
\end{align*}
where
 $R \eqdef I_m \otimes Q$.
Note that $R$ is unitary because
\begin{align*}
    RR^* =  (I_m \otimes Q)(I_m \otimes Q^*) = (I_mI_m)\otimes(QQ^*) = I_{mn^2};
\end{align*}
this implies that $M$ and $L$ have the same singular values.
\end{proof}

So now we have as a subproblem characterizing the singular values of a block matrix whose blocks are
diagonal.  
To express the the
characterization, 
it helps to reshape the nonzero elements of $L$ into a $m \times m \times n^2$ tensor $G$
as follows:
$
G_{c d w} = (D_{c d})_{w w}. 
$

\begin{theorem}\label{thm:L-svalue}
$
    \sigma(L) =
      \bigcup\limits_{w \in [n^2]} \sigma\left(G_{:,:,w} \right).
$
\end{theorem}
\begin{proof}
Choose an arbitrary $w \in [n^2]$, and a (scalar) singular value $\sigma$ of $G_{:,:,w}$ whose
left singular vector is $x$ and whose right singular vector is $y$, so that $G_{:,:,w}y = \sigma x$.  Recall that $e_w \in \reals^{n^2}$ is a standard basis vector. 

We claim that
$L (y \otimes e_w) = \sigma (x \otimes e_w).$
Since $D_{cd}$ is diagonal, $D_{cd} e_w = (D_{cd})_{ww} e_w = G_{cdw} e_w$. Thus we have $(L(y\otimes e_w))_c = \sum_{d \in [m]} D_{cd} y_d e_w = (\sum_{d \in [m]} G_{cdw} y_d) e_w = (G_{:,:,w}y)_c e_w = \sigma x_c e_w$, which shows that
\begin{align*}
L (y \otimes e_w) 
& = \left[  \begin{array}{cccc}
 D_{00}  & \dotsc & D_{0(m-1)} \\
  D_{10}  &  \dotsc & D_{1(m-1)}  \\
\vdots & \dotsc & \vdots   \\
 D_{(m-1)0}  & \dotsc & D_{(m-1)(m-1)}
\end{array} \right]  \left[  \begin{array}{c}
         y_0 e_w   \\
         \vdots   \\
         y_{m-1} e_w   \\
     \end{array}\right] 
    =  \left[  \begin{array}{c}
         \sigma x_0 e_w   \\
         \vdots   \\
         \sigma x_{m-1} e_w
     \end{array}\right] = \sigma (x \otimes e_w).
\end{align*}
If $\tilde{\sigma}$ is another singular value of $G_{:,:,w}$ with a left singular vector $\tilde{x}$ and
a right singular vector $\tilde{y}$, then $\langle (x \otimes e_w) , (\tilde{x} \otimes e_w) \rangle = \langle x,  \tilde{x} \rangle = 0$
and, similarly  $\langle (y \otimes e_w) , (\tilde{y} \otimes e_w)\rangle = 0$.  Also, 
$\langle (x \otimes e_w) , (x \otimes e_w) \rangle = 1$ and 
$\langle (y \otimes e_w) , (y \otimes e_w) \rangle = 1$.

For any $x$ and $\tilde{x}$, whether they are equal or not, if $w \neq \tilde{w}$,
then $\langle (x \otimes e_w) , (\tilde{x} \otimes e_{\tilde{w}}) \rangle = 0$, simply because
their non-zero components do not overlap.  

Thus, by taking the Kronecker product of each singular vector of $G_{:,:,w}$ with $e_w$ and
assembling the results for various $w$,
we may form a singular value decomposition of $L$ whose singular values are $\cup_{w \in [n^2]} \sigma(G_{:,:,w})$.
This completes the proof.
\end{proof}

Using Lemmas~\ref{lemma:m-l-svalue} and Theorem~\ref{thm:L-svalue}, 
we are now ready to prove Theorem~\ref{thm:conv.sv}.

{\bf Proof (of Theorem~\ref{thm:conv.sv})}.
Recall that, for each input channel $c$ and output channel $d$, the diagonal elements of
$D_{c,d}$ are the eigenvalues of $B_{c,d}$.  By Theorem~\ref{thm:singular_values}, this
means that the diagonal elements of $D_{c,d}$ are
\begin{equation}
\label{e:diagonal.elements}
\{ (F^T K_{:,:,c,d} F)_{u, v} : u, v \in [n] \}.
\end{equation}
The elements of (\ref{e:diagonal.elements}) map to the diagonal elements of $D_{c,d}$ as follows:
$$
G_{c d w} = (D_{c,d})_{w w} = 
             (F^T K_{:,:,c,d} F)_{\floor{w/m}, w \mod m}
$$
and thus
$$
G_{:,:,w} = \left((F^T K_{:,:,c,d} F)_{\floor{w/m}, w \mod m}\right)_{c d}
$$
which in turn implies
$$
\bigcup_{w \in [n^2]} \sigma(G_{:,:,w})
  = \cup_{u \in [n], v \in [n]} \sigma \left(
               \left((F^T K_{:,:,c,d} F)_{u,v}\right)_{c d}
             \right).
$$
Applying Lemmas~\ref{lemma:m-l-svalue} and \ref{thm:L-svalue} completes the proof.
\qed

\section{Regularization}\label{sec:regularization}
We now show how to use the spectrum computed above to project a convolution onto the set of convolutions with bounded operator norm.  
We exploit the following key fact.
\begin{proposition}[\citet{lefkimmiatis2013hessian}, Proposition 1]
\label{p:projection}
Let $A \in \R^{n \times n}$, 
and let $A = UDV^\top$ be its singular value decomposition.  
Let $\tilde{A} = U\tilde{D}V^\top$, where,
for all $i \in [n]$,
$\tilde{D}_{ii} = \min(D_{ii}, c)$
and ${\cal B} = \{ X \ \mid\ ||X||_2 \leq c\}$.
Then $\tilde{A}$ is the projection of
$A$ onto $\cal B$; i.e. $\tilde{A} 
= \underset{X\in {\cal B}}{\argmin}{~||A - X||_F}$.  
\end{proposition}
This implies that the desired projection can be obtained by clipping
the singular values of linear transformation associated with a
convolutional layer to the interval $[0,c]$.  Note that the eigenvectors remained
the same in the proposition, hence the projected matrix is still
generated by a convolution. However, after the projection, the
resulting convolution neighborhood may become as large as $n \times
n$.  On the other hand, we can project this convolution onto the set of
convolutions with $k \times k$ neighborhoods, by zeroing out all other
coefficients.
NumPy code for this is
in Appendix~\ref{a:proj_code}.  

Repeatedly alternating the two projections would
give a point in the intersection of the two sets, i.e., a
$k \times k$ convolution with bounded operator norm
(\citet{CheneyGoldstein1959} Theorem 4, \citet{boyd2003alternating} Section 2), and
the projection onto that intersection could be found using
the more complicated
Dykstra's projection algorithm~\citep{boyle1986method}.  

When we wish to control the operator norm during an
iterative optimization process, however, repeating the alternating
projections does not seem to be worth it -- we found
that the first two projections already often produced a convolutional
layer with an operator norm close to the desired value.  Furthermore,
because SGD does not change the parameters very fast, we can
think of a given pair of projections as providing a warm start
for the next pair.  

In practice, we run the two projections once every few steps, thus
letting the projection alternate with the training. 
\section{Experiments}

First, we validated 
Theorem~\ref{thm:conv.sv}
with unit tests in which the output of the code given in the introduction is compared with 
evaluating the singular values by constructing the full matrix 
encoding
the linear transformation corresponding to the convolutional layer and computing its SVD.

\subsection{Timing}
\label{s:timing}
We generated 4D tensors of various shapes with random standard normal
values, and computed their singular values using the full matrix method, the NumPy code given above and the equivalent TensorFlow code. For small tensors, the NumPy code was faster than TensorFlow, but for larger tensors, the TensorFlow code was able to exploit the parallelism in the algorithm and run much faster on a GPU. The timing results are shown in Figure~\ref{fig:timing}.

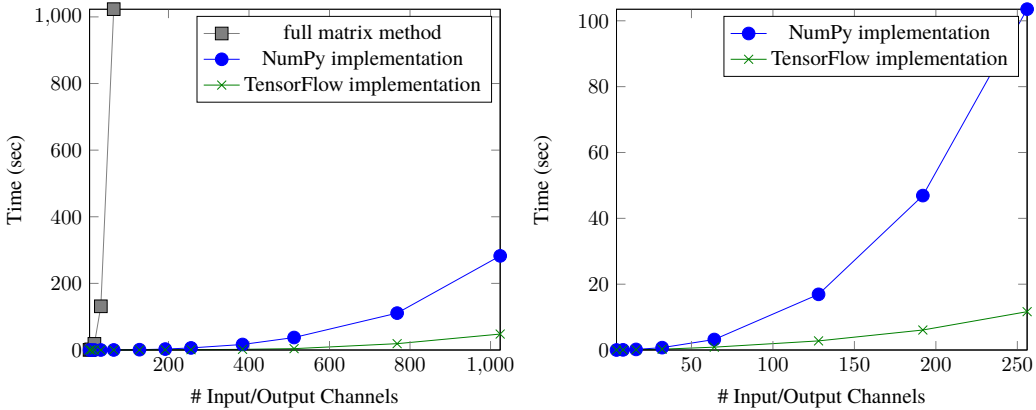
\begin{figure}[t]
\centering
\resizebox{1.0\textwidth}{!}{
%
%
%
%
\begin{tikzpicture}

\begin{axis}[
xlabel={\# Input/Output Channels},
ylabel={Time (sec)},
xmin=4, xmax=1024,
ymin=0.00339412689208984, ymax=1023.2387149334,
axis on top,
legend entries={{full matrix method},{NumPy implementation},{TensorFlow implementation}}
]
\addplot [lightgray!66.928104575163388!black, mark=square*, mark size=3, mark options={draw=black}]
coordinates {
(4,0.312762022018433)
(8,1.9580078125)
(16,18.4324049949646)
(32,131.246283054352)
(64,1023.2387149334)

};
\addplot [blue, mark=*, mark size=3, mark options={draw=blue}]
coordinates {
(4,0.00373387336730957)
(8,0.00339412689208984)
(16,0.0108850002288818)
(32,0.0450549125671387)
(64,0.191354036331177)
(128,0.976063966751099)
(192,2.79039406776428)
(256,6.26813077926636)
(384,16.6031079292297)
(512,37.5141170024872)
(768,110.76019692421)
(1024,282.571004867554)

};
\addplot [green!50.196078431372548!black, mark=x, mark size=3, mark options={draw=green!50.196078431372548!black}]
coordinates {
(4,0.0236310958862305)
(8,0.0079340934753418)
(16,0.0237419605255127)
(32,0.0203769207000732)
(64,0.10454797744751)
(128,0.278985977172852)
(192,0.503947019577026)
(256,0.909852981567383)
(384,2.0352098941803)
(512,4.07918190956116)
(768,19.0364699363708)
(1024,47.8946151733398)

};
\path [draw=black, fill opacity=0] (axis cs:4,1023.2387149334)--(axis cs:1024,1023.2387149334);

\path [draw=black, fill opacity=0] (axis cs:1024,0.00339412689208984)--(axis cs:1024,1023.2387149334);

\path [draw=black, fill opacity=0] (axis cs:4,0.00339412689208984)--(axis cs:4,1023.2387149334);

\path [draw=black, fill opacity=0] (axis cs:4,0.00339412689208984)--(axis cs:1024,0.00339412689208984);

\end{axis}

\end{tikzpicture}

%
%
%
%
\begin{tikzpicture}

\begin{axis}[
xlabel={\# Input/Output Channels},
ylabel={Time (sec)},
xmin=4, xmax=256,
ymin=0.012117862701416, ymax=103.537392139435,
axis on top,
legend entries={{NumPy implementation},{TensorFlow implementation}}
]
\addplot [blue, mark=*, mark size=3, mark options={draw=blue}]
coordinates {
(4,0.0168490409851074)
(8,0.0569248199462891)
(16,0.190770864486694)
(32,0.731719017028809)
(64,3.202388048172)
(128,16.9218420982361)
(192,46.9111888408661)
(256,103.537392139435)

};
\addplot [green!50.196078431372548!black, mark=x, mark size=3, mark options={draw=green!50.196078431372548!black}]
coordinates {
(4,0.012117862701416)
(8,0.0340299606323242)
(16,0.0897300243377686)
(32,0.261460065841675)
(64,0.877080917358398)
(128,2.78098917007446)
(192,6.08673691749573)
(256,11.6553330421448)

};
\path [draw=black, fill opacity=0] (axis cs:4,103.537392139435)--(axis cs:256,103.537392139435);

\path [draw=black, fill opacity=0] (axis cs:256,0.012117862701416)--(axis cs:256,103.537392139435);

\path [draw=black, fill opacity=0] (axis cs:4,0.012117862701416)--(axis cs:4,103.537392139435);

\path [draw=black, fill opacity=0] (axis cs:4,0.012117862701416)--(axis cs:256,0.012117862701416);

\end{axis}

\end{tikzpicture}
}
\caption{Time used to compute singular values. 
The left graph is for a $3 \times 3$ convolution on a $16 \times 16$ image with the number of input/output channels on the $x$-axis. The right graph is for a $11 \times 11$ convolution on a $64 \times 64$ image (no curve for full matrix method is shown as this method could not complete in a reasonable time for these inputs).}
\label{fig:timing}
\end{figure}

\subsection{Regularization}
\label{s:cifar}

We next explored the effect of regularizing the convolutional layers
by clipping their operator norms as described in Section~\ref{sec:regularization}. We ran the CIFAR-10 benchmark with a
standard 32 layer residual network with 2.4M training parameters;
\citep{he2016identity}. This network reached a test error rate of
$6.2\%$ after 250 epochs, using a learning rate schedule determined by
a grid search (shown by the gray plot in Figure~\ref{fig:cifar10}). We
then evaluated an algorithm that, every 100 steps, clipped the norms
of the convolutional layers to various different values between 0.1
and 3.0.  As expected, clipping to 2.5 and 3.0 had little impact on
the performance, since the norms of the convolutional layers were
between 2.5 and 2.8. Clipping to 0.1 yielded a surprising $6.7\%$ test
error, whereas clipping to 0.5 and 1.0 yielded test errors of $5.3\%$
and $5.5\%$ respectively (shown in Figure~\ref{fig:cifar10}).  A plot
of test error against training time is provided in Figure~\ref{fig:cifar10.error_vs_time} in Appendix~\ref{a:cifar10.error_vs_time}, 
showing that the projections did not slow down the training very much.

\begin{figure}[H]
\centering
\resizebox{1.0\textwidth}{!}{
\begin{tikzpicture}

\begin{axis}[
xlabel={\# Epochs},
ylabel={Loss},
xmin=10, xmax=250,
ymin=0, ymax=1.8,
axis on top,
xtick={0,50,100,150,200,250},
xticklabels={,50,100,150,200,250},
legend entries={{no clipping},{clip at 0.5},{clip at 1.0}}
]
\addplot [lightgray!66.928104575163388!black]
coordinates {
(10,1.73948461538461)
(11,1.67064358974359)
(12,1.59473846153846)
(13,1.54967435897436)
(14,1.48464102564103)
(15,1.43760512820513)
(16,1.41122307692308)
(17,1.34314615384615)
(18,1.31084615384615)
(19,1.27927948717949)
(20,1.21540769230769)
(21,1.19603333333333)
(22,1.17923076923077)
(23,1.10597435897436)
(24,1.10168205128205)
(25,1.0552282051282)
(26,1.02159487179487)
(27,0.998564102564102)
(28,0.968628205128205)
(29,0.924028205128205)
(30,0.928258974358974)
(31,0.896748717948718)
(32,0.867310256410256)
(33,0.845251282051282)
(34,0.813689743589743)
(35,0.809707692307692)
(36,0.771046153846154)
(37,0.750684615384615)
(38,0.751694871794871)
(39,0.712433333333333)
(40,0.712630769230769)
(41,0.706620512820513)
(42,0.687176923076923)
(43,0.679497435897436)
(44,0.65175641025641)
(45,0.633102564102564)
(46,0.633233333333333)
(47,0.613941025641026)
(48,0.598330769230769)
(49,0.6075)
(50,0.592251282051282)
(51,0.570661538461538)
(52,0.555307692307692)
(53,0.546025641025641)
(54,0.538902564102564)
(55,0.53095641025641)
(56,0.519197435897436)
(57,0.506130769230769)
(58,0.504823076923077)
(59,0.492182051282051)
(60,0.478241025641026)
(61,0.482294871794872)
(62,0.482571794871795)
(63,0.473941025641026)
(64,0.468289743589743)
(65,0.44615641025641)
(66,0.449348717948718)
(67,0.439607692307692)
(68,0.42904358974359)
(69,0.427923076923077)
(70,0.432005128205128)
(71,0.4147)
(72,0.397982051282051)
(73,0.414538461538462)
(74,0.395835897435897)
(75,0.389248717948718)
(76,0.402102564102564)
(77,0.389330769230769)
(78,0.380338461538461)
(79,0.382087179487179)
(80,0.364951282051282)
(81,0.37595641025641)
(82,0.367158974358974)
(83,0.373305128205128)
(84,0.376348717948718)
(85,0.367658974358974)
(86,0.359694871794872)
(87,0.351692307692308)
(88,0.358407692307692)
(89,0.344471794871795)
(90,0.346228205128205)
(91,0.341371794871795)
(92,0.334384615384615)
(93,0.34295641025641)
(94,0.341687179487179)
(95,0.322748717948718)
(96,0.324092307692308)
(97,0.344487179487179)
(98,0.332994871794872)
(99,0.314353846153846)
(100,0.322907692307692)
(101,0.320802564102564)
(102,0.322128205128205)
(103,0.32104358974359)
(104,0.327494871794872)
(105,0.3133)
(106,0.31491282051282)
(107,0.312515384615385)
(108,0.303605128205128)
(109,0.303612820512821)
(110,0.2982)
(111,0.291807692307692)
(112,0.316625641025641)
(113,0.300079487179487)
(114,0.292253846153846)
(115,0.299917948717949)
(116,0.298817948717949)
(117,0.297802564102564)
(118,0.286676923076923)
(119,0.296812820512821)
(120,0.255451282051282)
(121,0.222548717948718)
(122,0.214074358974359)
(123,0.210592307692308)
(124,0.206125641025641)
(125,0.20155641025641)
(126,0.200882051282051)
(127,0.197341025641026)
(128,0.192894871794872)
(129,0.194241025641026)
(130,0.190717948717949)
(131,0.187694871794872)
(132,0.186869230769231)
(133,0.184458974358974)
(134,0.180023076923077)
(135,0.179725641025641)
(136,0.177846153846154)
(137,0.175430769230769)
(138,0.172789743589744)
(139,0.172620512820513)
(140,0.171187179487179)
(141,0.169902564102564)
(142,0.167374358974359)
(143,0.168669230769231)
(144,0.164807692307692)
(145,0.163069230769231)
(146,0.162802564102564)
(147,0.15864358974359)
(148,0.157846153846154)
(149,0.155507692307692)
(150,0.156923076923077)
(151,0.15601282051282)
(152,0.152946153846154)
(153,0.152253846153846)
(154,0.15161282051282)
(155,0.15094358974359)
(156,0.147479487179487)
(157,0.147174358974359)
(158,0.146776923076923)
(159,0.144535897435897)
(160,0.142889743589744)
(161,0.143010256410256)
(162,0.144179487179487)
(163,0.142976923076923)
(164,0.142612820512821)
(165,0.142966666666667)
(166,0.143189743589744)
(167,0.14315641025641)
(168,0.141261538461538)
(169,0.142507692307692)
(170,0.142828205128205)
(171,0.141661538461538)
(172,0.140671794871795)
(173,0.140558974358974)
(174,0.141305128205128)
(175,0.142887179487179)
(176,0.141135897435897)
(177,0.140748717948718)
(178,0.140448717948718)
(179,0.139920512820513)
(180,0.140664102564103)
(181,0.140279487179487)
(182,0.1405)
(183,0.140025641025641)
(184,0.139589743589744)
(185,0.139153846153846)
(186,0.13951282051282)
(187,0.139523076923077)
(188,0.139517948717949)
(189,0.138828205128205)
(190,0.139510256410256)
(191,0.139366666666667)
(192,0.139174358974359)
(193,0.139210256410256)
(194,0.139451282051282)
(195,0.139930769230769)
(196,0.139320512820513)
(197,0.139048717948718)
(198,0.139971794871795)
(199,0.139123076923077)
(200,0.137928205128205)
(201,0.138558974358974)
(202,0.137497435897436)
(203,0.137989743589744)
(204,0.138830769230769)
(205,0.137707692307692)
(206,0.13714358974359)
(207,0.138502564102564)
(208,0.137474358974359)
(209,0.13705641025641)
(210,0.137992307692308)
(211,0.136833333333333)
(212,0.136702564102564)
(213,0.136274358974359)
(214,0.136892307692308)
(215,0.136251282051282)
(216,0.137287179487179)
(217,0.137215384615385)
(218,0.136464102564103)
(219,0.137448717948718)
(220,0.136407692307692)
(221,0.136520512820513)
(222,0.136458974358974)
(223,0.136294871794872)
(224,0.137846153846154)
(225,0.136784615384615)
(226,0.136823076923077)
(227,0.136294871794872)
(228,0.137069230769231)
(229,0.1375)
(230,0.136528205128205)
(231,0.136764102564103)
(232,0.13641282051282)
(233,0.13585641025641)
(234,0.1362)
(235,0.136582051282051)
(236,0.135912820512821)
(237,0.136494871794872)
(238,0.136179487179487)
(239,0.136876923076923)
(240,0.136410256410256)
(241,0.136764102564103)
(242,0.136002564102564)
(243,0.136676923076923)
(244,0.136764102564103)
(245,0.135892307692308)
(246,0.13541282051282)
(247,0.135902564102564)
(248,0.136194871794872)
(249,0.135887179487179)
(250,0.135616)

};
\addplot [color = green!50.0!black, style=dashed]
coordinates {
(10,0.637)
(11,0.659269230769231)
(12,0.654071794871795)
(13,0.592835897435897)
(14,0.621748717948718)
(15,0.591610256410256)
(16,0.567210256410256)
(17,0.582594871794872)
(18,0.576015384615384)
(19,0.57574358974359)
(20,0.567223076923077)
(21,0.508420512820513)
(22,0.532133333333333)
(23,0.49781282051282)
(24,0.531464102564102)
(25,0.529546153846154)
(26,0.509292307692307)
(27,0.530969230769231)
(28,0.521435897435897)
(29,0.516066666666667)
(30,0.509215384615384)
(31,0.506858974358974)
(32,0.487107692307692)
(33,0.497782051282051)
(34,0.476905128205128)
(35,0.478174358974359)
(36,0.463220512820513)
(37,0.458782051282051)
(38,0.45685641025641)
(39,0.474997435897436)
(40,0.440038461538461)
(41,0.463135897435897)
(42,0.473397435897436)
(43,0.4724)
(44,0.464848717948718)
(45,0.456569230769231)
(46,0.444584615384615)
(47,0.445217948717948)
(48,0.447102564102564)
(49,0.441069230769231)
(50,0.461876923076923)
(51,0.433271794871795)
(52,0.441841025641026)
(53,0.439082051282051)
(54,0.426664102564102)
(55,0.457169230769231)
(56,0.473153846153846)
(57,0.427082051282051)
(58,0.434151282051282)
(59,0.4439)
(60,0.434525641025641)
(61,0.434374358974359)
(62,0.42355641025641)
(63,0.420097435897436)
(64,0.426648717948718)
(65,0.449610256410256)
(66,0.418920512820513)
(67,0.43444358974359)
(68,0.409751282051282)
(69,0.428610256410256)
(70,0.425048717948718)
(71,0.421941025641026)
(72,0.434817948717949)
(73,0.422633333333333)
(74,0.422169230769231)
(75,0.421105128205128)
(76,0.435741025641026)
(77,0.404735897435897)
(78,0.401717948717949)
(79,0.404487179487179)
(80,0.415461538461538)
(81,0.417571794871795)
(82,0.407171794871795)
(83,0.40954358974359)
(84,0.437464102564102)
(85,0.417802564102564)
(86,0.41594358974359)
(87,0.418720512820513)
(88,0.422333333333333)
(89,0.407025641025641)
(90,0.410584615384615)
(91,0.4309)
(92,0.402415384615384)
(93,0.389497435897436)
(94,0.4032)
(95,0.401205128205128)
(96,0.401438461538461)
(97,0.413653846153846)
(98,0.404989743589743)
(99,0.401151282051282)
(100,0.392633333333333)
(101,0.423694871794872)
(102,0.414969230769231)
(103,0.403048717948718)
(104,0.39944358974359)
(105,0.420623076923077)
(106,0.425507692307692)
(107,0.415535897435897)
(108,0.401376923076923)
(109,0.414797435897436)
(110,0.40781282051282)
(111,0.411838461538461)
(112,0.407569230769231)
(113,0.400294871794872)
(114,0.402420512820513)
(115,0.414682051282051)
(116,0.385882051282051)
(117,0.381507692307692)
(118,0.400294871794872)
(119,0.399384615384615)
(120,0.304471794871795)
(121,0.214753846153846)
(122,0.205279487179487)
(123,0.188420512820513)
(124,0.189823076923077)
(125,0.186671794871795)
(126,0.173025641025641)
(127,0.174915384615385)
(128,0.165741025641026)
(129,0.172497435897436)
(130,0.171223076923077)
(131,0.163258974358974)
(132,0.170115384615385)
(133,0.168087179487179)
(134,0.158017948717949)
(135,0.16605641025641)
(136,0.159789743589744)
(137,0.162415384615385)
(138,0.158620512820513)
(139,0.154120512820513)
(140,0.144164102564103)
(141,0.155605128205128)
(142,0.147207692307692)
(143,0.157974358974359)
(144,0.152120512820513)
(145,0.145630769230769)
(146,0.149851282051282)
(147,0.149248717948718)
(148,0.149710256410256)
(149,0.148987179487179)
(150,0.143471794871795)
(151,0.146761538461538)
(152,0.143771794871795)
(153,0.151312820512821)
(154,0.154841025641026)
(155,0.149341025641026)
(156,0.149397435897436)
(157,0.114723076923077)
(158,0.1066)
(159,0.10294358974359)
(160,0.100061538461538)
(161,0.100697435897436)
(162,0.0969282051282051)
(163,0.0963230769230769)
(164,0.0986589743589744)
(165,0.0916589743589743)
(166,0.0928076923076923)
(167,0.0926692307692308)
(168,0.0930948717948718)
(169,0.0927538461538461)
(170,0.0922102564102564)
(171,0.0894846153846154)
(172,0.0908717948717948)
(173,0.088851282051282)
(174,0.0904538461538461)
(175,0.0896358974358974)
(176,0.0884025641025641)
(177,0.0871794871794871)
(178,0.0876153846153846)
(179,0.0871384615384615)
(180,0.0874282051282051)
(181,0.0854025641025641)
(182,0.0873897435897436)
(183,0.0861897435897436)
(184,0.0862025641025641)
(185,0.0866769230769231)
(186,0.0867846153846153)
(187,0.0857435897435898)
(188,0.0857384615384615)
(189,0.0863923076923077)
(190,0.0841179487179487)
(191,0.0851717948717948)
(192,0.0849076923076923)
(193,0.0836923076923077)
(194,0.0838692307692308)
(195,0.0832128205128205)
(196,0.0839564102564102)
(197,0.0842564102564102)
(198,0.0824102564102564)
(199,0.0834205128205128)
(200,0.082774358974359)
(201,0.0835333333333333)
(202,0.0823384615384615)
(203,0.0824820512820513)
(204,0.0818794871794872)
(205,0.0827076923076923)
(206,0.0825589743589743)
(207,0.0818102564102564)
(208,0.0814871794871795)
(209,0.0804743589743589)
(210,0.0802435897435897)
(211,0.0820307692307692)
(212,0.0821358974358974)
(213,0.081525641025641)
(214,0.0812358974358974)
(215,0.0805153846153846)
(216,0.0814641025641025)
(217,0.081625641025641)
(218,0.0807538461538461)
(219,0.0802384615384615)
(220,0.0809948717948718)
(221,0.0822358974358974)
(222,0.0822076923076923)
(223,0.0798384615384615)
(224,0.0814923076923077)
(225,0.0808333333333333)
(226,0.0812717948717948)
(227,0.0809)
(228,0.0807076923076923)
(229,0.0803923076923077)
(230,0.0803666666666667)
(231,0.0810820512820513)
(232,0.0811794871794872)
(233,0.0820435897435897)
(234,0.0812564102564102)
(235,0.0806051282051282)
(236,0.0810666666666666)
(237,0.0815794871794872)
(238,0.0808871794871795)
(239,0.0822)
(240,0.0819717948717949)
(241,0.0801)
(242,0.0809153846153846)
(243,0.0814410256410256)
(244,0.0809384615384615)
(245,0.0808179487179487)
(246,0.080974358974359)
(247,0.0804051282051282)
(248,0.0809948717948718)
(249,0.0813461538461538)
(250,0.081596)

};
\addplot [blue, thick, style=dotted]
coordinates {
(10,0.805853846153846)
(11,0.759443589743589)
(12,0.747648717948718)
(13,0.697041025641026)
(14,0.702176923076923)
(15,0.669294871794872)
(16,0.655364102564102)
(17,0.640638461538461)
(18,0.606692307692308)
(19,0.586805128205128)
(20,0.596551282051282)
(21,0.557374358974359)
(22,0.557194871794872)
(23,0.548333333333333)
(24,0.534002564102564)
(25,0.52434358974359)
(26,0.525779487179487)
(27,0.505897435897436)
(28,0.494784615384615)
(29,0.506489743589744)
(30,0.489961538461538)
(31,0.480238461538461)
(32,0.460079487179487)
(33,0.470525641025641)
(34,0.470305128205128)
(35,0.437864102564102)
(36,0.447638461538461)
(37,0.43104358974359)
(38,0.42771282051282)
(39,0.419861538461538)
(40,0.414315384615384)
(41,0.402025641025641)
(42,0.406397435897436)
(43,0.387823076923077)
(44,0.386951282051282)
(45,0.391046153846154)
(46,0.390971794871795)
(47,0.367966666666667)
(48,0.375071794871795)
(49,0.372841025641026)
(50,0.366833333333333)
(51,0.363689743589744)
(52,0.353341025641026)
(53,0.3627)
(54,0.358525641025641)
(55,0.351969230769231)
(56,0.358597435897436)
(57,0.377753846153846)
(58,0.351464102564102)
(59,0.34971282051282)
(60,0.346484615384615)
(61,0.36624358974359)
(62,0.339676923076923)
(63,0.330084615384615)
(64,0.336282051282051)
(65,0.321551282051282)
(66,0.332584615384615)
(67,0.328602564102564)
(68,0.315569230769231)
(69,0.33554358974359)
(70,0.32574358974359)
(71,0.306410256410256)
(72,0.337964102564103)
(73,0.323087179487179)
(74,0.319153846153846)
(75,0.325066666666667)
(76,0.324430769230769)
(77,0.298094871794872)
(78,0.31435641025641)
(79,0.306489743589743)
(80,0.298305128205128)
(81,0.316723076923077)
(82,0.297)
(83,0.293061538461538)
(84,0.293953846153846)
(85,0.298492307692308)
(86,0.304017948717949)
(87,0.299951282051282)
(88,0.292910256410256)
(89,0.290453846153846)
(90,0.281284615384615)
(91,0.305948717948718)
(92,0.2919)
(93,0.290887179487179)
(94,0.288479487179487)
(95,0.29415641025641)
(96,0.290423076923077)
(97,0.302405128205128)
(98,0.301335897435897)
(99,0.286594871794872)
(100,0.285110256410256)
(101,0.282261538461538)
(102,0.28201282051282)
(103,0.291228205128205)
(104,0.290958974358974)
(105,0.276138461538461)
(106,0.281730769230769)
(107,0.285823076923077)
(108,0.29314358974359)
(109,0.266289743589744)
(110,0.288497435897436)
(111,0.284287179487179)
(112,0.281782051282051)
(113,0.293889743589743)
(114,0.268820512820513)
(115,0.287410256410256)
(116,0.28521282051282)
(117,0.29161282051282)
(118,0.268153846153846)
(119,0.287830769230769)
(120,0.223794871794872)
(121,0.179471794871795)
(122,0.170738461538461)
(123,0.167289743589744)
(124,0.159253846153846)
(125,0.158346153846154)
(126,0.153284615384615)
(127,0.154212820512821)
(128,0.149551282051282)
(129,0.147176923076923)
(130,0.146751282051282)
(131,0.145346153846154)
(132,0.142348717948718)
(133,0.139466666666667)
(134,0.13964358974359)
(135,0.136584615384615)
(136,0.135094871794872)
(137,0.135176923076923)
(138,0.134166666666667)
(139,0.134117948717949)
(140,0.131189743589744)
(141,0.130376923076923)
(142,0.13124358974359)
(143,0.126846153846154)
(144,0.127902564102564)
(145,0.128074358974359)
(146,0.128315384615385)
(147,0.124115384615385)
(148,0.124038461538462)
(149,0.124948717948718)
(150,0.123284615384615)
(151,0.123007692307692)
(152,0.122433333333333)
(153,0.119587179487179)
(154,0.120482051282051)
(155,0.11984358974359)
(156,0.119835897435897)
(157,0.116117948717949)
(158,0.113117948717949)
(159,0.112410256410256)
(160,0.111166666666667)
(161,0.110617948717949)
(162,0.111923076923077)
(163,0.109602564102564)
(164,0.110535897435897)
(165,0.109628205128205)
(166,0.11125641025641)
(167,0.108692307692308)
(168,0.108897435897436)
(169,0.108317948717949)
(170,0.108497435897436)
(171,0.109776923076923)
(172,0.108448717948718)
(173,0.108210256410256)
(174,0.108833333333333)
(175,0.108612820512821)
(176,0.107807692307692)
(177,0.106884615384615)
(178,0.107546153846154)
(179,0.106994871794872)
(180,0.107097435897436)
(181,0.107392307692308)
(182,0.107841025641026)
(183,0.107015384615385)
(184,0.106302564102564)
(185,0.106753846153846)
(186,0.106489743589744)
(187,0.106864102564103)
(188,0.10671282051282)
(189,0.105982051282051)
(190,0.105597435897436)
(191,0.105784615384615)
(192,0.105505128205128)
(193,0.106771794871795)
(194,0.105505128205128)
(195,0.105638461538462)
(196,0.105523076923077)
(197,0.105676923076923)
(198,0.105810256410256)
(199,0.105902564102564)
(200,0.105061538461538)
(201,0.104935897435897)
(202,0.10534358974359)
(203,0.104761538461538)
(204,0.104930769230769)
(205,0.104969230769231)
(206,0.104425641025641)
(207,0.104823076923077)
(208,0.104879487179487)
(209,0.104397435897436)
(210,0.104625641025641)
(211,0.104287179487179)
(212,0.104541025641026)
(213,0.10435641025641)
(214,0.10431282051282)
(215,0.104212820512821)
(216,0.104653846153846)
(217,0.104)
(218,0.104620512820513)
(219,0.104289743589744)
(220,0.104535897435897)
(221,0.103679487179487)
(222,0.104015384615385)
(223,0.104453846153846)
(224,0.103825641025641)
(225,0.104292307692308)
(226,0.104351282051282)
(227,0.103825641025641)
(228,0.103535897435897)
(229,0.104215384615385)
(230,0.103894871794872)
(231,0.103384615384615)
(232,0.104053846153846)
(233,0.103515384615385)
(234,0.103566666666667)
(235,0.104020512820513)
(236,0.103835897435897)
(237,0.103812820512821)
(238,0.103676923076923)
(239,0.104074358974359)
(240,0.103584615384615)
(241,0.103364102564103)
(242,0.103594871794872)
(243,0.103502564102564)
(244,0.105576923076923)
(245,0.103907692307692)
(246,0.103594871794872)
(247,0.104238461538462)
(248,0.103853846153846)
(249,0.103528205128205)
(250,0.104224)

};

\path [draw=black, fill opacity=0] (axis cs:10,1.8)--(axis cs:250,1.8);

\path [draw=black, fill opacity=0] (axis cs:250,0)--(axis cs:250,1.8);

\path [draw=black, fill opacity=0] (axis cs:10,0)--(axis cs:10,1.8);

\path [draw=black, fill opacity=0] (axis cs:10,0)--(axis cs:250,0);

\end{axis}

\end{tikzpicture}

%
%
%
%
\begin{tikzpicture}

\begin{axis}[
xlabel={\# Epochs},
ylabel={Test Error},
xmin=10, xmax=249,
ymin=0, ymax=0.4,
axis on top,
legend entries={{no clipping},{clip at 0.5},{clip at 1.0}}
]
\addplot [white!50.196078431372548!black]
coordinates {
(10,0.292583333333333)
(11,0.290866666666667)
(12,0.293833333333333)
(13,0.275116666666667)
(14,0.2593)
(15,0.2586)
(16,0.25405)
(17,0.239666666666667)
(18,0.235233333333333)
(19,0.212183333333333)
(20,0.222316666666667)
(21,0.2154)
(22,0.206066666666667)
(23,0.206233333333333)
(24,0.192366666666667)
(25,0.192183333333333)
(26,0.177733333333333)
(27,0.179216666666667)
(28,0.177466666666667)
(29,0.1699)
(30,0.172833333333333)
(31,0.168583333333333)
(32,0.173016666666667)
(33,0.178266666666667)
(34,0.184016666666667)
(35,0.194233333333333)
(36,0.189366666666667)
(37,0.194716666666667)
(38,0.18715)
(39,0.174616666666667)
(40,0.1809)
(41,0.172166666666667)
(42,0.1696)
(43,0.17005)
(44,0.170466666666667)
(45,0.169033333333333)
(46,0.161316666666667)
(47,0.164533333333333)
(48,0.16825)
(49,0.16425)
(50,0.161)
(51,0.169183333333333)
(52,0.169933333333333)
(53,0.172266666666667)
(54,0.17785)
(55,0.177083333333333)
(56,0.1781)
(57,0.166166666666667)
(58,0.160416666666667)
(59,0.148533333333333)
(60,0.1379)
(61,0.1348)
(62,0.133183333333333)
(63,0.142483333333333)
(64,0.1474)
(65,0.148866666666667)
(66,0.16315)
(67,0.164416666666667)
(68,0.163683333333333)
(69,0.160916666666667)
(70,0.153116666666667)
(71,0.1567)
(72,0.149433333333333)
(73,0.146483333333333)
(74,0.148533333333333)
(75,0.150783333333333)
(76,0.1522)
(77,0.162066666666667)
(78,0.154466666666667)
(79,0.155616666666667)
(80,0.1547)
(81,0.15035)
(82,0.151)
(83,0.140583333333333)
(84,0.149533333333333)
(85,0.147366666666667)
(86,0.144833333333333)
(87,0.144466666666667)
(88,0.141233333333333)
(89,0.141716666666667)
(90,0.13485)
(91,0.13665)
(92,0.133866666666667)
(93,0.1403)
(94,0.138583333333333)
(95,0.137016666666667)
(96,0.132033333333333)
(97,0.129783333333333)
(98,0.1337)
(99,0.128183333333333)
(100,0.139933333333333)
(101,0.138633333333333)
(102,0.13915)
(103,0.13875)
(104,0.135783333333333)
(105,0.13415)
(106,0.1262)
(107,0.129433333333333)
(108,0.1272)
(109,0.130816666666667)
(110,0.131816666666667)
(111,0.132166666666667)
(112,0.132533333333333)
(113,0.127233333333333)
(114,0.132566666666667)
(115,0.1307)
(116,0.134916666666667)
(117,0.136866666666667)
(118,0.133583333333333)
(119,0.133033333333333)
(120,0.122466666666667)
(121,0.112566666666667)
(122,0.100833333333333)
(123,0.0862666666666667)
(124,0.0791666666666666)
(125,0.07205)
(126,0.0708166666666666)
(127,0.07025)
(128,0.0700333333333333)
(129,0.0704666666666666)
(130,0.0693333333333333)
(131,0.06935)
(132,0.0693333333333333)
(133,0.0693333333333333)
(134,0.0692166666666666)
(135,0.0686833333333333)
(136,0.06865)
(137,0.06875)
(138,0.06885)
(139,0.06955)
(140,0.0695)
(141,0.07035)
(142,0.07075)
(143,0.0706833333333333)
(144,0.0707166666666666)
(145,0.0711666666666667)
(146,0.0712833333333333)
(147,0.0705333333333333)
(148,0.07085)
(149,0.071)
(150,0.0716)
(151,0.0714666666666667)
(152,0.0717333333333333)
(153,0.0717833333333333)
(154,0.07205)
(155,0.0721666666666667)
(156,0.0716)
(157,0.0705333333333333)
(158,0.0696666666666667)
(159,0.0689666666666667)
(160,0.06785)
(161,0.06665)
(162,0.06595)
(163,0.0655)
(164,0.0650666666666667)
(165,0.06475)
(166,0.0642166666666666)
(167,0.0638)
(168,0.0635)
(169,0.06335)
(170,0.0633333333333333)
(171,0.0632833333333333)
(172,0.0633833333333333)
(173,0.0635166666666667)
(174,0.06355)
(175,0.0633833333333333)
(176,0.0632166666666666)
(177,0.0631166666666667)
(178,0.0630333333333333)
(179,0.0629833333333333)
(180,0.06275)
(181,0.0626833333333333)
(182,0.0625833333333334)
(183,0.0625333333333334)
(184,0.0624666666666667)
(185,0.0623833333333333)
(186,0.0624166666666667)
(187,0.06245)
(188,0.0623666666666667)
(189,0.0622666666666667)
(190,0.0620833333333333)
(191,0.0620666666666667)
(192,0.0620333333333333)
(193,0.0619333333333333)
(194,0.0618833333333333)
(195,0.0618)
(196,0.06175)
(197,0.0615333333333333)
(198,0.0615)
(199,0.0615833333333333)
(200,0.0616333333333333)
(201,0.06165)
(202,0.0617)
(203,0.0618333333333333)
(204,0.06175)
(205,0.0616833333333333)
(206,0.0617166666666667)
(207,0.0618333333333333)
(208,0.0619)
(209,0.0617833333333333)
(210,0.06185)
(211,0.0617666666666667)
(212,0.0616333333333333)
(213,0.0614666666666667)
(214,0.0613333333333333)
(215,0.0612666666666667)
(216,0.0612)
(217,0.0611833333333334)
(218,0.0612166666666667)
(219,0.0612166666666667)
(220,0.0612166666666667)
(221,0.0612666666666667)
(222,0.0612333333333334)
(223,0.0612166666666667)
(224,0.0611833333333334)
(225,0.06125)
(226,0.0613166666666667)
(227,0.06135)
(228,0.0613333333333333)
(229,0.0613833333333333)
(230,0.0613833333333333)
(231,0.0614)
(232,0.0613833333333333)
(233,0.0613833333333333)
(234,0.0614666666666667)
(235,0.0614666666666667)
(236,0.0615)
(237,0.0615)
(238,0.0614833333333333)
(239,0.0614833333333333)
(240,0.0614666666666667)
(241,0.0615)
(242,0.0615)
(243,0.0615166666666667)
(244,0.0615833333333333)
(245,0.0616166666666666)
(246,0.06165)
(247,0.0616333333333333)
(248,0.0616833333333333)
(249,0.0617)

};
\addplot [green!50.0!black,style=dashed]
coordinates {
(10,0.743583333333333)
(11,0.653466666666667)
(12,0.580133333333333)
(13,0.517816666666667)
(14,0.446416666666667)
(15,0.366616666666667)
(16,0.303433333333333)
(17,0.309016666666667)
(18,0.304483333333333)
(19,0.299666666666667)
(20,0.3232)
(21,0.36205)
(22,0.358666666666667)
(23,0.349833333333333)
(24,0.339833333333333)
(25,0.344266666666667)
(26,0.3157)
(27,0.2823)
(28,0.285416666666667)
(29,0.284766666666667)
(30,0.27955)
(31,0.268216666666667)
(32,0.30215)
(33,0.324083333333333)
(34,0.329716666666667)
(35,0.330816666666667)
(36,0.322033333333333)
(37,0.314983333333333)
(38,0.280216666666667)
(39,0.253966666666667)
(40,0.249083333333333)
(41,0.243733333333333)
(42,0.272016666666667)
(43,0.284816666666667)
(44,0.27995)
(45,0.310666666666667)
(46,0.308016666666667)
(47,0.3054)
(48,0.285033333333333)
(49,0.277116666666667)
(50,0.280516666666667)
(51,0.254033333333333)
(52,0.241016666666667)
(53,0.2505)
(54,0.25705)
(55,0.250533333333333)
(56,0.250716666666667)
(57,0.260016666666667)
(58,0.267983333333333)
(59,0.258933333333333)
(60,0.25365)
(61,0.2537)
(62,0.252916666666667)
(63,0.238966666666667)
(64,0.246816666666667)
(65,0.268683333333333)
(66,0.2775)
(67,0.2802)
(68,0.27515)
(69,0.273083333333333)
(70,0.289633333333333)
(71,0.271266666666667)
(72,0.270016666666667)
(73,0.271583333333333)
(74,0.275)
(75,0.283416666666667)
(76,0.253366666666667)
(77,0.242766666666667)
(78,0.253383333333333)
(79,0.267066666666667)
(80,0.262833333333333)
(81,0.258016666666667)
(82,0.252533333333333)
(83,0.255483333333333)
(84,0.243083333333333)
(85,0.224933333333333)
(86,0.247883333333333)
(87,0.2428)
(88,0.260933333333333)
(89,0.26125)
(90,0.245766666666667)
(91,0.2445)
(92,0.23415)
(93,0.260166666666667)
(94,0.25045)
(95,0.2486)
(96,0.261)
(97,0.273716666666667)
(98,0.269833333333333)
(99,0.244516666666667)
(100,0.23655)
(101,0.241416666666667)
(102,0.235566666666667)
(103,0.218483333333333)
(104,0.210033333333333)
(105,0.22535)
(106,0.2405)
(107,0.240933333333333)
(108,0.250133333333333)
(109,0.253516666666667)
(110,0.254633333333333)
(111,0.258066666666667)
(112,0.247816666666667)
(113,0.254766666666667)
(114,0.2391)
(115,0.237483333333333)
(116,0.228183333333333)
(117,0.23015)
(118,0.231166666666667)
(119,0.21595)
(120,0.208216666666667)
(121,0.193816666666667)
(122,0.187066666666667)
(123,0.146816666666667)
(124,0.1263)
(125,0.1144)
(126,0.104816666666667)
(127,0.0970166666666666)
(128,0.0939666666666666)
(129,0.0954666666666666)
(130,0.100083333333333)
(131,0.0996)
(132,0.0997333333333333)
(133,0.1012)
(134,0.0976333333333333)
(135,0.0978333333333333)
(136,0.0940333333333333)
(137,0.0922833333333333)
(138,0.0921333333333333)
(139,0.0913333333333333)
(140,0.09395)
(141,0.0963)
(142,0.09525)
(143,0.09795)
(144,0.10095)
(145,0.102233333333333)
(146,0.103833333333333)
(147,0.102583333333333)
(148,0.110133333333333)
(149,0.107933333333333)
(150,0.10585)
(151,0.1065)
(152,0.1063)
(153,0.10565)
(154,0.09935)
(155,0.0993)
(156,0.0940666666666667)
(157,0.08795)
(158,0.0799833333333333)
(159,0.0721)
(160,0.06555)
(161,0.0599166666666666)
(162,0.0574)
(163,0.0561666666666667)
(164,0.0553666666666666)
(165,0.0549833333333333)
(166,0.0547)
(167,0.0545)
(168,0.0543833333333333)
(169,0.05395)
(170,0.05385)
(171,0.0538166666666667)
(172,0.0538666666666667)
(173,0.0540666666666667)
(174,0.0540333333333333)
(175,0.05395)
(176,0.0537833333333333)
(177,0.0535666666666667)
(178,0.0535166666666667)
(179,0.0532666666666667)
(180,0.0531666666666667)
(181,0.0531833333333334)
(182,0.0532333333333333)
(183,0.0533333333333333)
(184,0.0533666666666667)
(185,0.0533)
(186,0.05325)
(187,0.05305)
(188,0.0528166666666667)
(189,0.05265)
(190,0.05255)
(191,0.0527666666666667)
(192,0.0527)
(193,0.0529166666666667)
(194,0.053)
(195,0.0530833333333333)
(196,0.0531166666666667)
(197,0.05305)
(198,0.0533833333333333)
(199,0.0534166666666667)
(200,0.05345)
(201,0.0534666666666667)
(202,0.0534333333333333)
(203,0.0533333333333333)
(204,0.05305)
(205,0.0531166666666666)
(206,0.0533)
(207,0.0532)
(208,0.0531166666666666)
(209,0.0530833333333333)
(210,0.0532333333333333)
(211,0.0529166666666667)
(212,0.05255)
(213,0.0526166666666667)
(214,0.0528)
(215,0.05295)
(216,0.0528833333333333)
(217,0.0530833333333333)
(218,0.05325)
(219,0.0533)
(220,0.0532333333333333)
(221,0.0531333333333333)
(222,0.0531166666666667)
(223,0.0530833333333333)
(224,0.0530666666666667)
(225,0.0530666666666667)
(226,0.05305)
(227,0.0530833333333333)
(228,0.05305)
(229,0.0530166666666667)
(230,0.0531)
(231,0.0530166666666667)
(232,0.053)
(233,0.05295)
(234,0.0529166666666666)
(235,0.0530333333333333)
(236,0.053)
(237,0.053)
(238,0.05295)
(239,0.0529666666666667)
(240,0.0530333333333333)
(241,0.0529833333333333)
(242,0.0529833333333333)
(243,0.0531333333333333)
(244,0.0532833333333333)
(245,0.0533666666666667)
(246,0.0534)
(247,0.0534)
(248,0.0534166666666667)
(249,0.0532833333333333)

};
\addplot [blue, thick,style=dotted]
coordinates {
(10,0.413583333333333)
(11,0.39915)
(12,0.3601)
(13,0.311483333333333)
(14,0.28965)
(15,0.280666666666667)
(16,0.249616666666667)
(17,0.214766666666667)
(18,0.20655)
(19,0.206333333333333)
(20,0.203433333333333)
(21,0.19895)
(22,0.188366666666667)
(23,0.192516666666667)
(24,0.181116666666667)
(25,0.177533333333333)
(26,0.1825)
(27,0.1791)
(28,0.179183333333333)
(29,0.1723)
(30,0.1796)
(31,0.185666666666667)
(32,0.175866666666667)
(33,0.182916666666667)
(34,0.177983333333333)
(35,0.174033333333333)
(36,0.171283333333333)
(37,0.17575)
(38,0.182933333333333)
(39,0.196783333333333)
(40,0.211016666666667)
(41,0.217816666666667)
(42,0.222416666666667)
(43,0.209683333333333)
(44,0.20755)
(45,0.190783333333333)
(46,0.178316666666667)
(47,0.181933333333333)
(48,0.177783333333333)
(49,0.181333333333333)
(50,0.18305)
(51,0.181316666666667)
(52,0.189616666666667)
(53,0.184816666666667)
(54,0.179266666666667)
(55,0.178966666666667)
(56,0.1714)
(57,0.175183333333333)
(58,0.168283333333333)
(59,0.173816666666667)
(60,0.185816666666667)
(61,0.189233333333333)
(62,0.195033333333333)
(63,0.185916666666667)
(64,0.1904)
(65,0.1885)
(66,0.180716666666667)
(67,0.181616666666667)
(68,0.17995)
(69,0.194066666666667)
(70,0.194033333333333)
(71,0.189683333333333)
(72,0.201616666666667)
(73,0.19265)
(74,0.20525)
(75,0.1938)
(76,0.191983333333333)
(77,0.197616666666667)
(78,0.183066666666667)
(79,0.1852)
(80,0.178233333333333)
(81,0.171933333333333)
(82,0.16205)
(83,0.14865)
(84,0.1526)
(85,0.154083333333333)
(86,0.1427)
(87,0.15085)
(88,0.14935)
(89,0.157033333333333)
(90,0.149316666666667)
(91,0.146233333333333)
(92,0.158683333333333)
(93,0.152366666666667)
(94,0.165133333333333)
(95,0.1577)
(96,0.168033333333333)
(97,0.171066666666667)
(98,0.159183333333333)
(99,0.16115)
(100,0.153133333333333)
(101,0.155166666666667)
(102,0.145683333333333)
(103,0.147666666666667)
(104,0.1479)
(105,0.147583333333333)
(106,0.141916666666667)
(107,0.136883333333333)
(108,0.139416666666667)
(109,0.133283333333333)
(110,0.135583333333333)
(111,0.13215)
(112,0.138233333333333)
(113,0.142066666666667)
(114,0.142183333333333)
(115,0.144316666666667)
(116,0.1468)
(117,0.152333333333333)
(118,0.14775)
(119,0.163266666666667)
(120,0.153633333333333)
(121,0.1398)
(122,0.124016666666667)
(123,0.110766666666667)
(124,0.10115)
(125,0.0735666666666667)
(126,0.0693333333333333)
(127,0.0670333333333333)
(128,0.0656833333333333)
(129,0.0648666666666667)
(130,0.0639333333333333)
(131,0.0640166666666667)
(132,0.0643166666666666)
(133,0.0648)
(134,0.0649833333333333)
(135,0.0644666666666666)
(136,0.06505)
(137,0.0640166666666666)
(138,0.0638833333333333)
(139,0.0638833333333333)
(140,0.06435)
(141,0.0649)
(142,0.06405)
(143,0.06445)
(144,0.0651166666666666)
(145,0.0663166666666666)
(146,0.06685)
(147,0.0697333333333333)
(148,0.0699)
(149,0.07015)
(150,0.0707333333333333)
(151,0.0706833333333333)
(152,0.0706833333333333)
(153,0.06805)
(154,0.0719)
(155,0.0742)
(156,0.0732333333333333)
(157,0.07175)
(158,0.07005)
(159,0.0687333333333333)
(160,0.0639333333333333)
(161,0.0602833333333333)
(162,0.05865)
(163,0.0578166666666667)
(164,0.0575833333333333)
(165,0.0572666666666667)
(166,0.0570333333333333)
(167,0.0569833333333333)
(168,0.0570166666666667)
(169,0.0570333333333333)
(170,0.0570166666666667)
(171,0.05705)
(172,0.0572666666666667)
(173,0.0573166666666667)
(174,0.0572)
(175,0.0572833333333333)
(176,0.0570833333333333)
(177,0.0571166666666667)
(178,0.0568666666666667)
(179,0.0567166666666667)
(180,0.05665)
(181,0.0565)
(182,0.0564166666666667)
(183,0.0562166666666667)
(184,0.0559166666666666)
(185,0.05565)
(186,0.0555833333333333)
(187,0.0555166666666666)
(188,0.0555333333333333)
(189,0.0556666666666666)
(190,0.0558333333333333)
(191,0.0560833333333333)
(192,0.0560833333333333)
(193,0.0559666666666666)
(194,0.0559833333333333)
(195,0.0558)
(196,0.0557833333333333)
(197,0.0556666666666667)
(198,0.0555333333333333)
(199,0.0555)
(200,0.0553666666666667)
(201,0.0552666666666667)
(202,0.0551333333333333)
(203,0.0551333333333333)
(204,0.0551666666666667)
(205,0.0552166666666667)
(206,0.0551833333333333)
(207,0.0553)
(208,0.0554833333333333)
(209,0.0554)
(210,0.0553666666666667)
(211,0.0553833333333333)
(212,0.0553833333333333)
(213,0.0553333333333333)
(214,0.0552333333333333)
(215,0.0552333333333333)
(216,0.0552833333333333)
(217,0.0552166666666667)
(218,0.05525)
(219,0.0552333333333333)
(220,0.05525)
(221,0.0553666666666667)
(222,0.0554166666666667)
(223,0.0554666666666667)
(224,0.0555)
(225,0.0555)
(226,0.0554833333333333)
(227,0.0554)
(228,0.0553833333333333)
(229,0.0554)
(230,0.0553833333333333)
(231,0.0554)
(232,0.0553833333333333)
(233,0.0554)
(234,0.0553666666666667)
(235,0.0553166666666667)
(236,0.0552833333333333)
(237,0.0552)
(238,0.0551)
(239,0.0550333333333333)
(240,0.0549833333333333)
(241,0.05495)
(242,0.0549333333333333)
(243,0.0549666666666667)
(244,0.0551333333333334)
(245,0.0551833333333334)
(246,0.0552333333333333)
(247,0.0552833333333333)
(248,0.0553333333333333)
(249,0.0553833333333333)

};

\path [draw=black, fill opacity=0] (axis cs:10,0.4)--(axis cs:249,0.4);

\path [draw=black, fill opacity=0] (axis cs:249,-6.93889390390723e-18)--(axis cs:249,0.4);

\path [draw=black, fill opacity=0] (axis cs:10,-6.93889390390723e-18)--(axis cs:249,-6.93889390390723e-18);

\path [draw=black, fill opacity=0] (axis cs:10,-6.93889390390723e-18)--(axis cs:10,0.4);

\end{axis}

\end{tikzpicture}

}
\caption{Training loss and test error for ResNet model~\citep{he2016identity} for CIFAR-10.}
\label{fig:cifar10}
\end{figure}
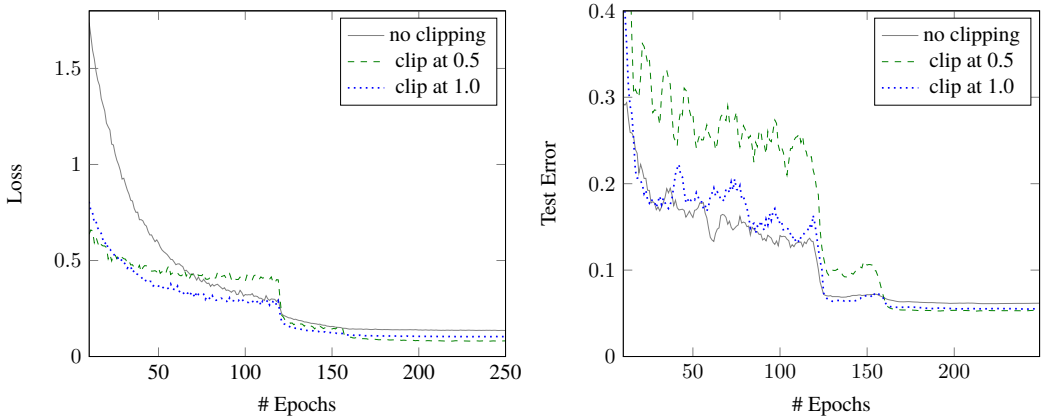

\subsection{Robustness to changes in hyperparameters}

The baseline algorithm studied in the previous subsection used
batch normalization.  Batch normalization tends to make the
network less sensitive to linear transformations with large
operator norms.  However, batch normalization includes
trainable scaling parameters (called $\gamma$ in the original
paper) that are applied after the normalization
step.  The existence of these parameters lead to a complicated
interaction between batch normalization and methods like ours
that act to control the norm of the linear transformation
applied before batch normalization.

Because the effect of regularizing the operator norm is more
easily understood in the absence of batch normalization, we
also
performed experiments with a baseline that
did not use batch normalization.   

Another possibility that we wanted to study was that using
a regularizer may make the process overall more stable, enabling
a larger learning rate.  We were generally interested in whether
operator-norm regularization made the training process more robust
to the choice of hyperparameters.

In one experiment, we started with the same baseline as the previous
subsection, but disabled batch normalization.  This baseline
started with a learning rate of 0.1, which was multiplied by a
factor 0.95 after every epoch.  We tried all combinations
of the following hyperparameters: 
(a) the norm of the ball projected onto (no projection, 0.5, 1.0, 1.5, 2.0);
(b) the initial learning rate (0.001, 0.003, 0.01, 0.03, 0.1);
(c) the minibatch size (32, 64);
(d) the number of epochs per decay of the learning rate (1,2,3).
We tried each of the $150$ combinations of the hyperparameters, trained
for 100 epochs, and measured the test error.  The results are
plotted in Figure~\ref{f:sweep_wo_batchnorm}.
\begin{figure}[h]
\centering
\begin{subfigure}[b]{0.45\textwidth}
\includegraphics[width=2.5in]{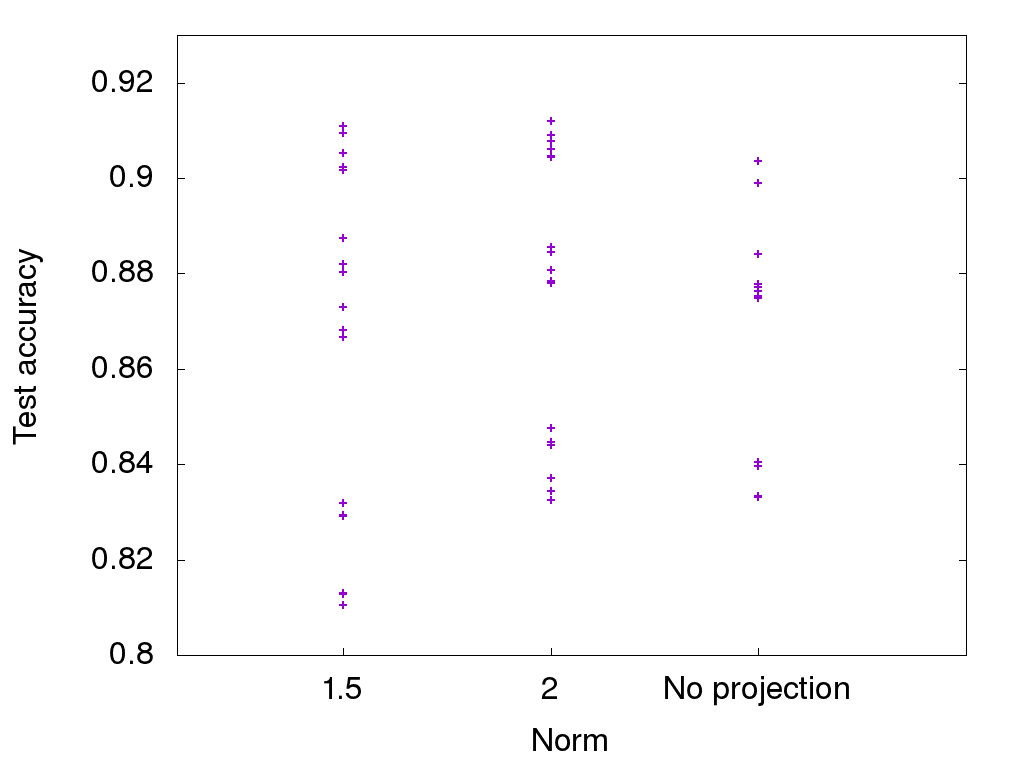}
\caption{Without batch normalization}
\label{f:sweep_wo_batchnorm}
\end{subfigure}
\quad
\begin{subfigure}[b]{0.45\textwidth}
\includegraphics[width=2.5in]{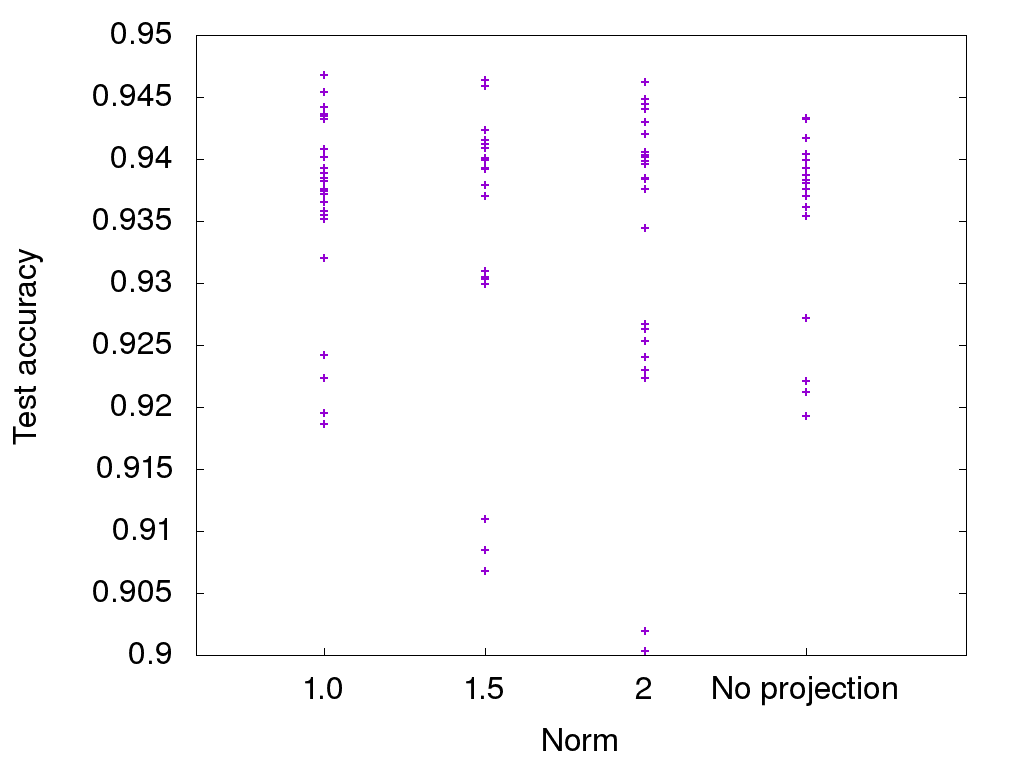}
\caption{With batch normalization}
\label{f:sweep_w_batchnorm}
\end{subfigure}
\caption{A scatterplot of the test errors obtained with different
hyperparameter combinations, and different operator-norm regularizers.}
\label{f:sweep}
\end{figure}
The operator norm regularization improved the best result, and
also made the process more robust to the choice of hyperparameters.

We conducted a similar experiment in the presence of batch normalization,
except using learning rates 0.01, 0.03, 0.1, 0.2, and 0.3.  
Those results are shown in Figure~\ref{f:sweep_w_batchnorm}.
Regularizing the operator norm helps, even in the presence of batch
normalization.

It appears that operator-norm regularization and batch normalization
are not redundant, and neither dominates the other.  We were surprised
by this.

%

\subsection{Comparison with reshaping $K$}\label{s:yoshida}
In Section~\ref{s:related} we mentioned that \citet{yoshida2017spectral} approximated the linear transformation induced by $K$ by reshaping $K$. This leads to an alternate regularization method --- we compute the spectrum of the reshaped $K$, and project it onto a ball using clipping, as above. We implemented this 
an experimented with it using
the same network and hyperparameters as
in Section~\ref{s:cifar} and found the following.
\begin{itemize}
\item We clipped the singular values of the reshaped $K$ every 100 steps. We tried various constants for the clipped value
(0.05, 0.1, 0.2, 0.5, 1.0), and found that the best accuracy we achieved,
using 0.2, was the same as the accuracy we achieved in Section~\ref{s:cifar}.
\item We clipped the singular values of the reshaped $K$ to these same values every step, and found that the best accuracy achieved was slightly worse than the accuracy achieved in the previous step. 
We observed similar behavior when we clipped norms using our method. 
\item Most surprisingly, we found that clipping norms by our method on
  a GPU was about 25\% faster than clipping the singular values of the
  reshaped $K$ --- when we clipped after every step, on the same
  machine, 10000 batches of CIFAR10 took 14490 seconds when we clipped
  the reshaped $K$, whereas they took 11004 seconds with our exact
  method! One possible explanation is parallelization --- clipping
  reshaped $K$ takes $O(m^3 k^2)$ flops, whereas our method does $m^2$
  FFTs, followed by $n^2$ $m \times m$ SVDs, which takes $O(m^3 n^2)$
  flops, but these can be parallelized and completed in as little as
  $O(n^2 \log n + m^3)$ time.
\end{itemize}

Clearly this is only one dataset, and the results may not generalize to other sets. However it does suggest that finding the full spectrum of the convolutional layer may be no worse than computing heuristic approximations, both in classification accuracy and speed.

 
\section{Acknowledgements}

We thank Tomer Koren, Nishal Shah, Yoram Singer and Chiyuan Zhang for
valuable conversations.



\bibliographystyle{plainnat}
\begin{small}

\end{small}

\appendix

\section{NumPy code for operator norm projection}
\label{a:proj_code}

\begin{verbatim}
def Clip_OperatorNorm(kernel, input_shape, clip_to):
  transform_coefficients = np.fft.fft2(kernel, input_shape, axes=[0, 1])
  U, D, V = np.linalg.svd(transform_coefficients, compute_uv=True, full_matrices=False)
  D_clipped = np.minimum(D, clip_to)
  if kernel.shape[2] > kernel.shape[3]:
    clipped_transform_coefficients = np.matmul(U, D_clipped[..., None] * V)
  else:
    clipped_transform_coefficients = np.matmul(U * D_clipped[..., None, :], V)
  clipped_kernel = np.fft.ifft2(clipped_transform_coefficients, axes=[0, 1]).real
  return clipped_kernel[np.ix_(*[range(d) for d in kernel.shape])]
\end{verbatim}

\section{Test error vs.\ training time}
\label{a:cifar10.error_vs_time}

Figure~\ref{fig:cifar10.error_vs_time} shows the plots of test error vs.\ training time
in our CIFAR-10 experiment.

\begin{figure}[H]
\centering
\resizebox{0.5\textwidth}{!}{
%
%
%
%
\begin{tikzpicture}

\begin{axis}[
xlabel={Time(seconds)},
ylabel={Test Error},
xmin=2241, xmax=56862,
ymin=0, ymax=0.4,
axis on top,
legend entries={{no clipping},{clip at 0.5},{clip at 1.0}}
]
\addplot [lightgray!66.928104575163388!black]
coordinates {
(2241,0.292583333333333)
(2467,0.290866666666667)
(2692,0.293833333333333)
(2917,0.275116666666667)
(3153,0.2593)
(3379,0.2586)
(3604,0.25405)
(3828,0.239666666666667)
(4053,0.235233333333333)
(4278,0.212183333333333)
(4504,0.222316666666667)
(4729,0.2154)
(4953,0.206066666666667)
(5178,0.206233333333333)
(5402,0.192366666666667)
(5626,0.192183333333333)
(5849,0.177733333333333)
(6065,0.179216666666667)
(6287,0.177466666666667)
(6510,0.1699)
(6725,0.172833333333333)
(6950,0.168583333333333)
(7173,0.173016666666667)
(7395,0.178266666666667)
(7618,0.184016666666667)
(7842,0.194233333333333)
(8065,0.189366666666667)
(8280,0.194716666666667)
(8513,0.18715)
(8736,0.174616666666667)
(8960,0.1809)
(9172,0.172166666666667)
(9394,0.1696)
(9605,0.17005)
(9827,0.170466666666667)
(10049,0.169033333333333)
(10272,0.161316666666667)
(10495,0.164533333333333)
(10710,0.16825)
(10932,0.16425)
(11144,0.161)
(11367,0.169183333333333)
(11590,0.169933333333333)
(11805,0.172266666666667)
(12028,0.17785)
(12250,0.177083333333333)
(12472,0.1781)
(12683,0.166166666666667)
(12906,0.160416666666667)
(13128,0.148533333333333)
(13349,0.1379)
(13573,0.1348)
(13796,0.133183333333333)
(14009,0.142483333333333)
(14231,0.1474)
(14454,0.148866666666667)
(14676,0.16315)
(14888,0.164416666666667)
(15108,0.163683333333333)
(15331,0.160916666666667)
(15554,0.153116666666667)
(15776,0.1567)
(15998,0.149433333333333)
(16209,0.146483333333333)
(16430,0.148533333333333)
(16652,0.150783333333333)
(16873,0.1522)
(17096,0.162066666666667)
(17319,0.154466666666667)
(17532,0.155616666666667)
(17754,0.1547)
(17977,0.15035)
(18199,0.151)
(18412,0.140583333333333)
(18636,0.149533333333333)
(18850,0.147366666666667)
(19071,0.144833333333333)
(19293,0.144466666666667)
(19515,0.141233333333333)
(19737,0.141716666666667)
(19949,0.13485)
(20172,0.13665)
(20396,0.133866666666667)
(20620,0.1403)
(20831,0.138583333333333)
(21056,0.137016666666667)
(21278,0.132033333333333)
(21501,0.129783333333333)
(21713,0.1337)
(21935,0.128183333333333)
(22158,0.139933333333333)
(22370,0.138633333333333)
(22593,0.13915)
(22804,0.13875)
(23025,0.135783333333333)
(23246,0.13415)
(23468,0.1262)
(23690,0.129433333333333)
(23901,0.1272)
(24123,0.130816666666667)
(24346,0.131816666666667)
(24557,0.132166666666667)
(24783,0.132533333333333)
(24996,0.127233333333333)
(25217,0.132566666666667)
(25431,0.1307)
(25654,0.134916666666667)
(25896,0.136866666666667)
(26119,0.133583333333333)
(26341,0.133033333333333)
(26556,0.122466666666667)
(26777,0.112566666666667)
(27000,0.100833333333333)
(27225,0.0862666666666667)
(27463,0.0791666666666666)
(27697,0.07205)
(27910,0.0708166666666666)
(28132,0.07025)
(28355,0.0700333333333333)
(28568,0.0704666666666666)
(28791,0.0693333333333333)
(29014,0.06935)
(29235,0.0693333333333333)
(29458,0.0693333333333333)
(29669,0.0692166666666666)
(29900,0.0686833333333333)
(30112,0.06865)
(30333,0.06875)
(30555,0.06885)
(30776,0.06955)
(31000,0.0695)
(31223,0.07035)
(31445,0.07075)
(31657,0.0706833333333333)
(31879,0.0707166666666666)
(32100,0.0711666666666667)
(32313,0.0712833333333333)
(32535,0.0705333333333333)
(32757,0.07085)
(32978,0.071)
(33200,0.0716)
(33413,0.0714666666666667)
(33634,0.0717333333333333)
(33856,0.0717833333333333)
(34077,0.07205)
(34288,0.0721666666666667)
(34511,0.0716)
(34732,0.0705333333333333)
(34944,0.0696666666666667)
(35165,0.0689666666666667)
(35386,0.06785)
(35599,0.06665)
(35821,0.06595)
(36042,0.0655)
(36263,0.0650666666666667)
(36485,0.06475)
(36696,0.0642166666666666)
(36917,0.0638)
(37139,0.0635)
(37361,0.06335)
(37582,0.0633333333333333)
(37803,0.0632833333333333)
(38014,0.0633833333333333)
(38244,0.0635166666666667)
(38455,0.06355)
(38676,0.0633833333333333)
(38901,0.0632166666666666)
(39122,0.0631166666666667)
(39344,0.0630333333333333)
(39566,0.0629833333333333)
(39787,0.06275)
(40008,0.0626833333333333)
(40230,0.0625833333333334)
(40441,0.0625333333333334)
(40660,0.0624666666666667)
(40882,0.0623833333333333)
(41103,0.0624166666666667)
(41324,0.06245)
(41544,0.0623666666666667)
(41765,0.0622666666666667)
(41987,0.0620833333333333)
(42198,0.0620666666666667)
(42420,0.0620333333333333)
(42640,0.0619333333333333)
(42862,0.0618833333333333)
(43085,0.0618)
(43305,0.06175)
(43525,0.0615333333333333)
(43746,0.0615)
(43967,0.0615833333333333)
(44187,0.0616333333333333)
(44408,0.06165)
(44619,0.0617)
(44841,0.0618333333333333)
(45061,0.06175)
(45283,0.0616833333333333)
(45504,0.0617166666666667)
(45715,0.0618333333333333)
(45936,0.0619)
(46158,0.0617833333333333)
(46379,0.06185)
(46590,0.0617666666666667)
(46812,0.0616333333333333)
(47033,0.0614666666666667)
(47254,0.0613333333333333)
(47465,0.0612666666666667)
(47686,0.0612)
(47909,0.0611833333333334)
(48120,0.0612166666666667)
(48350,0.0612166666666667)
(48561,0.0612166666666667)
(48781,0.0612666666666667)
(49000,0.0612333333333334)
(49221,0.0612166666666667)
(49442,0.0611833333333334)
(49663,0.06125)
(49873,0.0613166666666667)
(50095,0.06135)
(50317,0.0613333333333333)
(50548,0.0613833333333333)
(50768,0.0613833333333333)
(50979,0.0614)
(51200,0.0613833333333333)
(51422,0.0613833333333333)
(51642,0.0614666666666667)
(51864,0.0614666666666667)
(52085,0.0615)
(52306,0.0615)
(52527,0.0614833333333333)
(52738,0.0614833333333333)
(52959,0.0614666666666667)
(53181,0.0615)
(53401,0.0615)
(53622,0.0615166666666667)
(53832,0.0615833333333333)
(54055,0.0616166666666666)
(54275,0.06165)
(54496,0.0616333333333333)
(54718,0.0616833333333333)
(54930,0.0617)

};
\addplot [green!50.0!black,style=dashed]
coordinates {
(2324,0.743583333333333)
(2559,0.653466666666667)
(2793,0.580133333333333)
(3028,0.517816666666667)
(3261,0.446416666666667)
(3504,0.366616666666667)
(3740,0.303433333333333)
(3970,0.309016666666667)
(4212,0.304483333333333)
(4444,0.299666666666667)
(4678,0.3232)
(4910,0.36205)
(5141,0.358666666666667)
(5372,0.349833333333333)
(5606,0.339833333333333)
(5837,0.344266666666667)
(6067,0.3157)
(6299,0.2823)
(6530,0.285416666666667)
(6772,0.284766666666667)
(6995,0.27955)
(7226,0.268216666666667)
(7467,0.30215)
(7691,0.324083333333333)
(7923,0.329716666666667)
(8155,0.330816666666667)
(8377,0.322033333333333)
(8617,0.314983333333333)
(8838,0.280216666666667)
(9069,0.253966666666667)
(9300,0.249083333333333)
(9530,0.243733333333333)
(9760,0.272016666666667)
(9982,0.284816666666667)
(10213,0.27995)
(10443,0.310666666666667)
(10673,0.308016666666667)
(10903,0.3054)
(11133,0.285033333333333)
(11356,0.277116666666667)
(11587,0.280516666666667)
(11819,0.254033333333333)
(12050,0.241016666666667)
(12280,0.2505)
(12509,0.25705)
(12739,0.250533333333333)
(12959,0.250716666666667)
(13188,0.260016666666667)
(13418,0.267983333333333)
(13647,0.258933333333333)
(13877,0.25365)
(14097,0.2537)
(14327,0.252916666666667)
(14557,0.238966666666667)
(14787,0.246816666666667)
(15016,0.268683333333333)
(15235,0.2775)
(15466,0.2802)
(15696,0.27515)
(15915,0.273083333333333)
(16144,0.289633333333333)
(16373,0.271266666666667)
(16603,0.270016666666667)
(16822,0.271583333333333)
(17052,0.275)
(17282,0.283416666666667)
(17511,0.253366666666667)
(17740,0.242766666666667)
(17960,0.253383333333333)
(18189,0.267066666666667)
(18420,0.262833333333333)
(18651,0.258016666666667)
(18871,0.252533333333333)
(19101,0.255483333333333)
(19330,0.243083333333333)
(19562,0.224933333333333)
(19781,0.247883333333333)
(20010,0.2428)
(20229,0.260933333333333)
(20458,0.26125)
(20688,0.245766666666667)
(20918,0.2445)
(21139,0.23415)
(21367,0.260166666666667)
(21598,0.25045)
(21818,0.2486)
(22047,0.261)
(22279,0.273716666666667)
(22498,0.269833333333333)
(22726,0.244516666666667)
(22956,0.23655)
(23176,0.241416666666667)
(23396,0.235566666666667)
(23614,0.218483333333333)
(23834,0.210033333333333)
(24045,0.22535)
(24270,0.2405)
(24480,0.240933333333333)
(24702,0.250133333333333)
(24911,0.253516666666667)
(25131,0.254633333333333)
(25351,0.258066666666667)
(25570,0.247816666666667)
(25789,0.254766666666667)
(26019,0.2391)
(26238,0.237483333333333)
(26447,0.228183333333333)
(26667,0.23015)
(26887,0.231166666666667)
(27108,0.21595)
(27320,0.208216666666667)
(27548,0.193816666666667)
(27769,0.187066666666667)
(27989,0.146816666666667)
(28200,0.1263)
(28420,0.1144)
(28642,0.104816666666667)
(28852,0.0970166666666666)
(29072,0.0939666666666666)
(29293,0.0954666666666666)
(29513,0.100083333333333)
(29734,0.0996)
(29953,0.0997333333333333)
(30162,0.1012)
(30383,0.0976333333333333)
(30603,0.0978333333333333)
(30823,0.0940333333333333)
(31044,0.0922833333333333)
(31253,0.0921333333333333)
(31473,0.0913333333333333)
(31693,0.09395)
(31913,0.0963)
(32122,0.09525)
(32342,0.09795)
(32561,0.10095)
(32781,0.102233333333333)
(33000,0.103833333333333)
(33210,0.102583333333333)
(33429,0.110133333333333)
(33648,0.107933333333333)
(33868,0.10585)
(34087,0.1065)
(34306,0.1063)
(34525,0.10565)
(34746,0.09935)
(34956,0.0993)
(35175,0.0940666666666667)
(35394,0.08795)
(35614,0.0799833333333333)
(35833,0.0721)
(36043,0.06555)
(36262,0.0599166666666666)
(36482,0.0574)
(36702,0.0561666666666667)
(36921,0.0553666666666666)
(37141,0.0549833333333333)
(37350,0.0547)
(37570,0.0545)
(37790,0.0543833333333333)
(38009,0.05395)
(38218,0.05385)
(38437,0.0538166666666667)
(38657,0.0538666666666667)
(38876,0.0540666666666667)
(39095,0.0540333333333333)
(39314,0.05395)
(39534,0.0537833333333333)
(39754,0.0535666666666667)
(39963,0.0535166666666667)
(40182,0.0532666666666667)
(40402,0.0531666666666667)
(40622,0.0531833333333334)
(40841,0.0532333333333333)
(41050,0.0533333333333333)
(41269,0.0533666666666667)
(41488,0.0533)
(41707,0.05325)
(41927,0.05305)
(42137,0.0528166666666667)
(42356,0.05265)
(42565,0.05255)
(42784,0.0527666666666667)
(43004,0.0527)
(43224,0.0529166666666667)
(43444,0.053)
(43664,0.0530833333333333)
(43884,0.0531166666666667)
(44103,0.05305)
(44312,0.0533833333333333)
(44532,0.0534166666666667)
(44752,0.05345)
(44961,0.0534666666666667)
(45181,0.0534333333333333)
(45401,0.0533333333333333)
(45621,0.05305)
(45841,0.0531166666666666)
(46060,0.0533)
(46280,0.0532)
(46499,0.0531166666666666)
(46708,0.0530833333333333)
(46928,0.0532333333333333)
(47146,0.0529166666666667)
(47366,0.05255)
(47575,0.0526166666666667)
(47794,0.0528)
(48013,0.05295)
(48232,0.0528833333333333)
(48451,0.0530833333333333)
(48671,0.05325)
(48890,0.0533)
(49109,0.0532333333333333)
(49329,0.0531333333333333)
(49538,0.0531166666666667)
(49758,0.0530833333333333)
(49988,0.0530666666666667)
(50209,0.0530666666666667)
(50433,0.05305)
(50653,0.0530833333333333)
(50873,0.05305)
(51093,0.0530166666666667)
(51303,0.0531)
(51523,0.0530166666666667)
(51744,0.053)
(51954,0.05295)
(52174,0.0529166666666666)
(52393,0.0530333333333333)
(52613,0.053)
(52823,0.053)
(53042,0.05295)
(53263,0.0529666666666667)
(53482,0.0530333333333333)
(53704,0.0529833333333333)
(53924,0.0529833333333333)
(54134,0.0531333333333333)
(54355,0.0532833333333333)
(54575,0.0533666666666667)
(54795,0.0534)
(55005,0.0534)
(55225,0.0534166666666667)
(55444,0.0532833333333333)

};
\addplot [blue,thick, style=dotted]
coordinates {
(2323,0.413583333333333)
(2546,0.39915)
(2780,0.3601)
(3015,0.311483333333333)
(3249,0.28965)
(3481,0.280666666666667)
(3724,0.249616666666667)
(3957,0.214766666666667)
(4189,0.20655)
(4423,0.206333333333333)
(4656,0.203433333333333)
(4888,0.19895)
(5120,0.188366666666667)
(5352,0.192516666666667)
(5582,0.181116666666667)
(5813,0.177533333333333)
(6043,0.1825)
(6274,0.1791)
(6509,0.179183333333333)
(6742,0.1723)
(6974,0.1796)
(7205,0.185666666666667)
(7435,0.175866666666667)
(7667,0.182916666666667)
(7889,0.177983333333333)
(8121,0.174033333333333)
(8351,0.171283333333333)
(8583,0.17575)
(8815,0.182933333333333)
(9037,0.196783333333333)
(9269,0.211016666666667)
(9501,0.217816666666667)
(9731,0.222416666666667)
(9951,0.209683333333333)
(10173,0.20755)
(10403,0.190783333333333)
(10633,0.178316666666667)
(10862,0.181933333333333)
(11093,0.177783333333333)
(11325,0.181333333333333)
(11546,0.18305)
(11777,0.181316666666667)
(12009,0.189616666666667)
(12240,0.184816666666667)
(12460,0.179266666666667)
(12690,0.178966666666667)
(12920,0.1714)
(13150,0.175183333333333)
(13371,0.168283333333333)
(13601,0.173816666666667)
(13832,0.185816666666667)
(14052,0.189233333333333)
(14282,0.195033333333333)
(14513,0.185916666666667)
(14743,0.1904)
(14973,0.1885)
(15193,0.180716666666667)
(15423,0.181616666666667)
(15654,0.17995)
(15873,0.194066666666667)
(16103,0.194033333333333)
(16332,0.189683333333333)
(16562,0.201616666666667)
(16792,0.19265)
(17011,0.20525)
(17241,0.1938)
(17460,0.191983333333333)
(17690,0.197616666666667)
(17919,0.183066666666667)
(18149,0.1852)
(18379,0.178233333333333)
(18599,0.171933333333333)
(18830,0.16205)
(19059,0.14865)
(19288,0.1526)
(19508,0.154083333333333)
(19738,0.1427)
(19967,0.15085)
(20187,0.14935)
(20417,0.157033333333333)
(20646,0.149316666666667)
(20876,0.146233333333333)
(21105,0.158683333333333)
(21335,0.152366666666667)
(21564,0.165133333333333)
(21784,0.1577)
(22014,0.168033333333333)
(22243,0.171066666666667)
(22472,0.159183333333333)
(22692,0.16115)
(22922,0.153133333333333)
(23151,0.155166666666667)
(23371,0.145683333333333)
(23600,0.147666666666667)
(23831,0.1479)
(24061,0.147583333333333)
(24280,0.141916666666667)
(24512,0.136883333333333)
(24731,0.139416666666667)
(24961,0.133283333333333)
(25181,0.135583333333333)
(25401,0.13215)
(25630,0.138233333333333)
(25872,0.142066666666667)
(26091,0.142183333333333)
(26311,0.144316666666667)
(26530,0.1468)
(26750,0.152333333333333)
(26970,0.14775)
(27192,0.163266666666667)
(27417,0.153633333333333)
(27639,0.1398)
(27860,0.124016666666667)
(28081,0.110766666666667)
(28301,0.10115)
(28512,0.0735666666666667)
(28742,0.0693333333333333)
(28963,0.0670333333333333)
(29193,0.0656833333333333)
(29417,0.0648666666666667)
(29637,0.0639333333333333)
(29867,0.0640166666666667)
(30087,0.0643166666666666)
(30317,0.0648)
(30536,0.0649833333333333)
(30756,0.0644666666666666)
(30986,0.06505)
(31207,0.0640166666666666)
(31437,0.0638833333333333)
(31676,0.0638833333333333)
(31906,0.06435)
(32145,0.0649)
(32375,0.06405)
(32605,0.06445)
(32834,0.0651166666666666)
(33063,0.0663166666666666)
(33294,0.06685)
(33533,0.0697333333333333)
(33762,0.0699)
(33992,0.07015)
(34222,0.0707333333333333)
(34452,0.0706833333333333)
(34691,0.0706833333333333)
(34920,0.06805)
(35149,0.0719)
(35379,0.0742)
(35609,0.0732333333333333)
(35838,0.07175)
(36068,0.07005)
(36308,0.0687333333333333)
(36537,0.0639333333333333)
(36777,0.0602833333333333)
(37007,0.05865)
(37237,0.0578166666666667)
(37466,0.0575833333333333)
(37706,0.0572666666666667)
(37926,0.0570333333333333)
(38155,0.0569833333333333)
(38384,0.0570166666666667)
(38625,0.0570333333333333)
(38855,0.0570166666666667)
(39084,0.05705)
(39314,0.0572666666666667)
(39544,0.0573166666666667)
(39783,0.0572)
(40013,0.0572833333333333)
(40243,0.0570833333333333)
(40472,0.0571166666666667)
(40701,0.0568666666666667)
(40931,0.0567166666666667)
(41161,0.05665)
(41390,0.0565)
(41620,0.0564166666666667)
(41850,0.0562166666666667)
(42080,0.0559166666666666)
(42300,0.05565)
(42520,0.0555833333333333)
(42740,0.0555166666666666)
(42960,0.0555333333333333)
(43180,0.0556666666666666)
(43409,0.0558333333333333)
(43629,0.0560833333333333)
(43858,0.0560833333333333)
(44079,0.0559666666666666)
(44309,0.0559833333333333)
(44528,0.0558)
(44758,0.0557833333333333)
(44988,0.0556666666666667)
(45217,0.0555333333333333)
(45446,0.0555)
(45676,0.0553666666666667)
(45916,0.0552666666666667)
(46145,0.0551333333333333)
(46375,0.0551333333333333)
(46604,0.0551666666666667)
(46833,0.0552166666666667)
(47073,0.0551833333333333)
(47303,0.0553)
(47532,0.0554833333333333)
(47761,0.0554)
(47990,0.0553666666666667)
(48220,0.0553833333333333)
(48449,0.0553833333333333)
(48678,0.0553333333333333)
(48908,0.0552333333333333)
(49137,0.0552333333333333)
(49367,0.0552833333333333)
(49606,0.0552166666666667)
(49825,0.05525)
(50054,0.0552333333333333)
(50276,0.05525)
(50497,0.0553666666666667)
(50717,0.0554166666666667)
(50947,0.0554666666666667)
(51177,0.0555)
(51407,0.0555)
(51636,0.0554833333333333)
(51856,0.0554)
(52085,0.0553833333333333)
(52304,0.0554)
(52534,0.0553833333333333)
(52763,0.0554)
(52991,0.0553833333333333)
(53210,0.0554)
(53439,0.0553666666666667)
(53669,0.0553166666666667)
(53888,0.0552833333333333)
(54117,0.0552)
(54347,0.0551)
(54576,0.0550333333333333)
(54805,0.0549833333333333)
(55034,0.05495)
(55264,0.0549333333333333)
(55493,0.0549666666666667)
(55723,0.0551333333333334)
(55942,0.0551833333333334)
(56172,0.0552333333333333)
(56402,0.0552833333333333)
(56633,0.0553333333333333)
(56862,0.0553833333333333)

};

\path [draw=black, fill opacity=0] (axis cs:2241,-6.93889390390723e-18)--(axis cs:2241,0.4);

\path [draw=black, fill opacity=0] (axis cs:56862,-6.93889390390723e-18)--(axis cs:56862,0.4);

\path [draw=black, fill opacity=0] (axis cs:2241,0.4)--(axis cs:56862,0.4);

\path [draw=black, fill opacity=0] (axis cs:2241,-6.93889390390723e-18)--(axis cs:56862,-6.93889390390723e-18);

\end{axis}

\end{tikzpicture}

}
\caption{Test error vs.\ training time for ResNet model~\protect\citep{he2016identity} for CIFAR-10.}
\label{fig:cifar10.error_vs_time}
\end{figure}
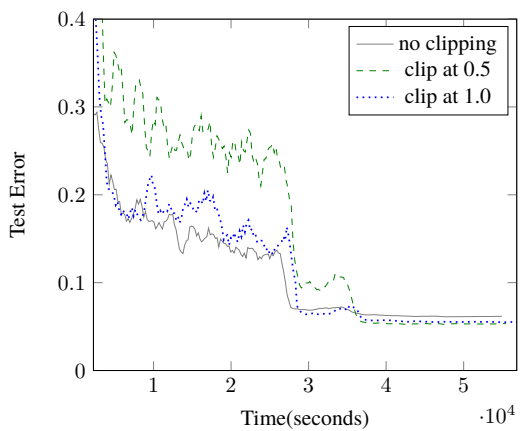

\section{The official pre-trained ResNet model}
\label{a:resnet}

\begin{figure}[h]
\centering
\includegraphics[width=2.5in]{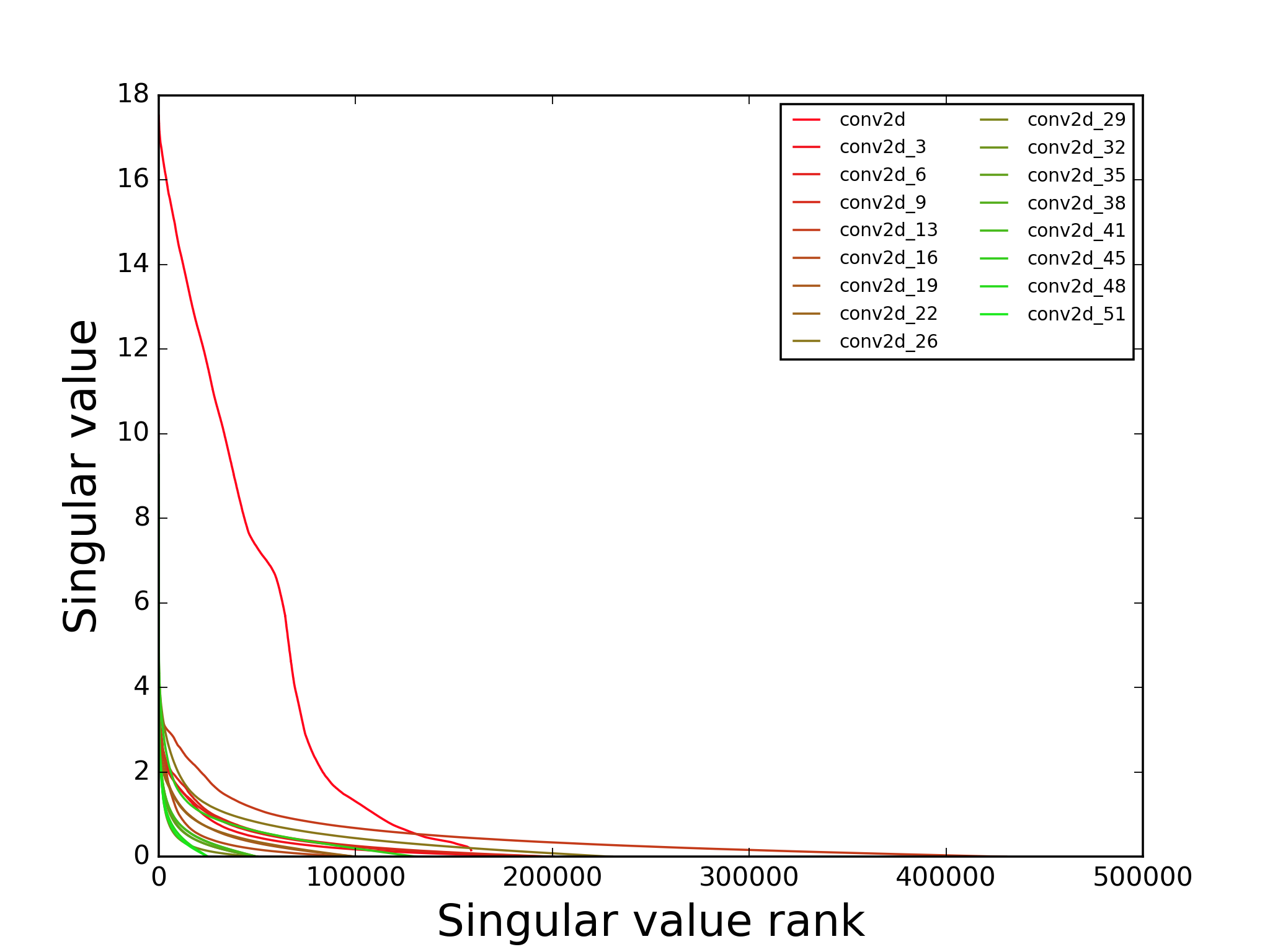}
\caption{Plot of the singular values of the linear operators associated with
the convolutional layers of the pretrained "ResNet V2" from the TensorFlow website.  }
\label{f:resnet_conv_svd_plot}
\end{figure}
The singular values of the convolutional layers from
the 
official ``Resnet V2'' 
pre-trained model
\citep{he2016identity}
are plotted in Figure~\ref{f:resnet_conv_svd_plot}.
The singular values
are ordered by value.  Only layers with kernels larger than $1 \times 1$
are plotted.  The curves are plotted with a mixture of red and green; layers closer to the input are
plotted with colors with a greater share of red.
The transformations with the largest operator norms are closest to the inputs.
As the data has undergone more rounds of processing, as we proceed through the layers, the number of non-negligible
singular values increases for a while, but at the end, it tapers off.  

In Figure~\ref{f:resnet_conv_svd_plot}, we plotted the singular values ordered by value. It can be observed that while singular values in the first layer are much larger than the rest, many layers have a lot of singular values that are pretty big.  For example, most of the layers have at least 10000 singular values that are at least 1.  To give a complementary view, Figure~\ref{fig:ratio} presents a plot of the ratios of the singular values in each layer with the largest singular value in that layer. We see that the effective rank of the convolutional layers is larger closer to the inputs.

\begin{figure}
\centering
\includegraphics[width=3in]{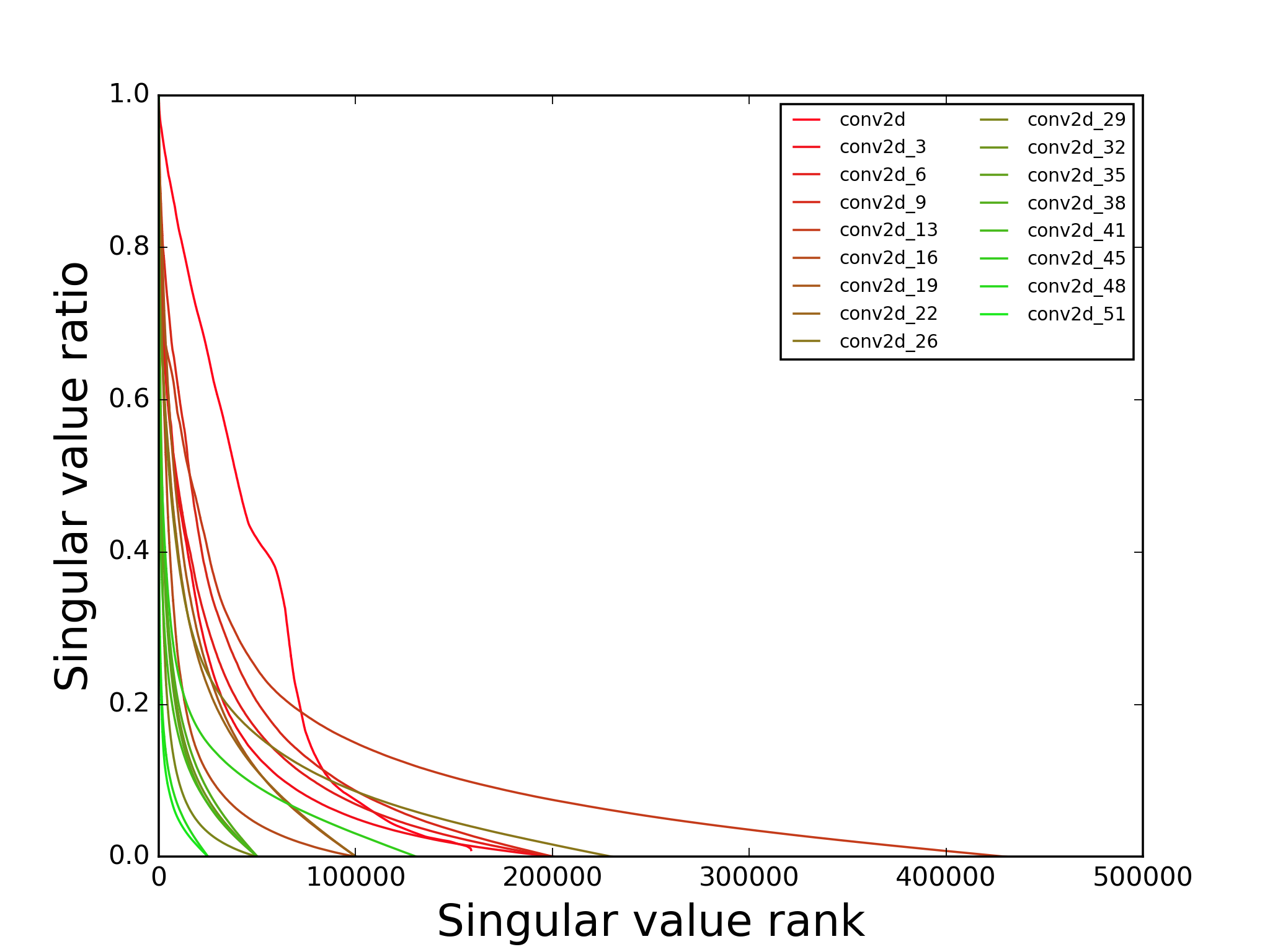}
\caption{Plot of the ratio of singular values to maximum singular value of the linear operators associated with
the convolutional layers of the pretrained "ResNet V2" from the TensorFlow website.  }
\label{fig:ratio}
\end{figure}

Figure~\ref{fig:ratio} shows that different convolutional layers have significantly different numbers of non-negligible singular values.  A question that may arise is to what extent this was due to the fact that different layers simply are of different sizes, so that the total number of their singular values, tiny or not, was different.  To look into this, instead of plotting the singular value ratios as a function of the rank of the singular values, as in the Figure~\ref{fig:ratio}, we normalized the values on the horizontal axis by dividing by the total number of singular values.  The result is shown in Figure~\ref{fig:sizenormalized}.

\begin{figure}
\centering
\includegraphics[width=3in]{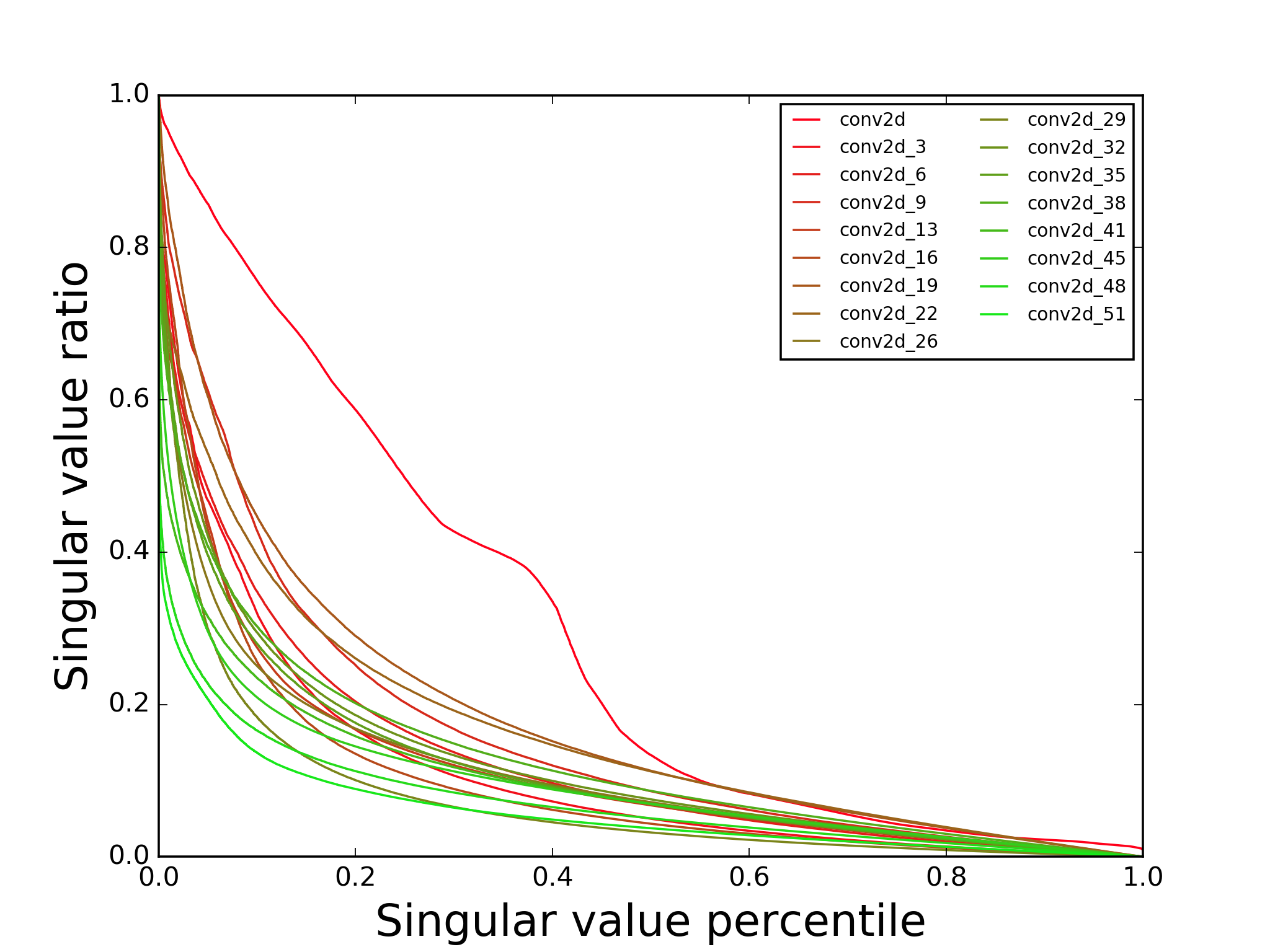}
\caption{Plot of the ratio of singular values to maximum singular value of the linear operators associated with
the convolutional layers of the pretrained "ResNet V2" normalized by size of the convolution.  }
\label{fig:sizenormalized} 
\end{figure}

  \end{document}